\documentclass{article}

\usepackage{PRIMEarxiv}

\usepackage[utf8]{inputenc} 
\usepackage[T1]{fontenc}    
\usepackage{hyperref}       
\usepackage{url}            
\usepackage{booktabs}       
\usepackage{amsfonts}       
\usepackage{nicefrac}       
\usepackage{microtype}      
\usepackage{lipsum}
\usepackage{fancyhdr}       
\usepackage{graphicx}       
\graphicspath{{media/}}     

\usepackage{amsmath,amsthm}

\newtheorem{lemma}{Lemma}
\newtheorem{definition}{Definition}

\title{Fast Knowledge Graph Completion using Graphics Processing Units
}

\author{
  Chun-Hee Lee\thanks{Corresponding Author} 
 , Dong-oh Kang, Hwa Jeon Song \\
  Intelligence Information Research Division, ETRI \\
  Daejeon, South Korea \\
  \texttt{\{ch.lee, dongoh, songhj\}@etri.re.kr} \\ 
}

\begin{document}
\maketitle

\begin{abstract}
Knowledge graphs can be used in many areas related to data semantics such as question-answering systems, knowledge based systems. However, the currently constructed knowledge graphs need to be complemented for better knowledge in terms of relations. It is called knowledge graph completion. To add new relations to the existing knowledge graph by using knowledge graph embedding models, we have to evaluate $N\times N \times R$ vector operations, where $N$ is the number of entities and $R$ is the number of relation types. It is very costly. 

In this paper, we provide an efficient knowledge graph completion framework on GPUs to get new relations using knowledge graph embedding vectors. In the proposed framework, we first define \emph{transformable to a metric space} and then provide a method to transform the knowledge graph completion problem into the similarity join problem for a model which is \emph{transformable to a metric space}. After that, to efficiently process the similarity join problem, we derive formulas using the properties of a metric space. Based on the formulas, we develop a fast knowledge graph completion algorithm. Finally, we experimentally show that our framework can efficiently process the knowledge graph completion problem.
\end{abstract}

\keywords{Knowledge Graph Embedding \and Knowedge Graph Completion \and TransE \and Similarity Join \and GPU Processing}

\section{Introduction}
Knowledge graphs can be used in a wide range of areas which require data semantics such as question-answering systems, semantic search systems, and knowledge based systems. A knowledge graph \cite{freebase,yago3,yago4} can be constructed using data sources from an open collaboration platform such as wikipedia or wikidata because an enormous amount of information can be gathered in the open collaboration platform. However, the constructed knowledge graph is still incomplete because there can exist a much larger number of potential relations (i.e., $N \times N \times R$, $N$: the number of entities, $R$: the number of relation types) compared with the number of relations in the existing knowledge graph and data sources from the open platform intrinsically cannot have all the information to connect the relations.

Therefore, we need to add a lot of missing relations (or links) to the knowledge graph. It is called \emph{knowledge graph completion}. Knowledge graph embedding is one of the most commonly used techniques for knowledge graph completion. Much work \cite{kge:transe,kge:transh,kge:transr,kge:stranse,kge:transf,kge:distmult,kge:complex,kge:simple,kge:conve,kge:acre,kge:rotate} has been studied in the literature to improve the accuracy of knowledge graph completion. 
However, most of the knowledge graph embedding studies do not tackle the running time of the knowledge graph completion. To find a meaningful link (i.e., to add a new relation to the knowledge graph), we should compute the score of each triplet $(head, relation, tail)$ and the number of triplets to be computed is very huge (i.e., $N \times N \times R$, $N$: is the number of nodes, $R$ is the number of relation types). For instance, if the number of nodes in the knowledge graph is one million and the number of relation types is one thousand, the number of scores to be evaluated is one quadrillion ($10^{6} \times 10^{6} \times 10^{3}=10^{15}$). Although processing the knowledge graph completion problem is prohibitive, this study has been overlooked in the literature of knowledge graph embedding.

In this paper, we focus on how to efficiently get triplets (i.e., $(head, relation, tail)$) with high scores from the given knowledge graph embedding vectors. Formally speaking, the problem we have to solve is to find a triplet $(\mathbf{h, r, t})$ such that $score(\mathbf{h, r, t}) \ge \epsilon$ for $\mathbf{h, t} \in E$ and $\mathbf{r} \in R$, where $E$ is the set of entity embedding vectors and $R$ is the set of relation embedding vectors. We propose a framework to efficiently solve the knowledge graph completion problem on GPUs because nowadays GPUs are commonly used and their computing capabilities are very powerful if we can organize computing jobs in a parallel way. To provide an efficient and systematic knowledge graph completion framework on GPUs, we first concentrate on \emph{TransE} \cite{kge:transe} which is one of the simple and powerful knowledge graph embedding models. Note that several papers \cite{kge:performance1,kge:performance2,kge:performance3} experimentally show that simple models such as TransE \cite{kge:transe} and DistMult \cite{kge:distmult} have a comparable performance compared to the recent knowledge graph embedding models in terms of accuracy if their hyperparameters are properly tuned. We introduce and define \emph{transformable to a metric space} to generalize our framework. Therefore, our framework can be applied to any knowledge graph embedding model which is \emph{transformable to a metric space}. 
 
The knowledge graph completion problem is similar to the similarity join problem \cite{sim:ekdb,sim:ego,sim:ego2,sim:quickjoin,sim:dboperator} in the database area. However, the approaches for the similarity join problem can not be applied to directly the knowledge graph completion problem because the similarity join problem has two input parameters whereas the knowledge graph completion problem has three input parameters. Also, the settings in the knowledge graph completion problem are different from those in the similarity join problem. In this paper, we provide a method to transform a knowledge graph embedding model, which is \emph{transformable to a metric space}, into the similarity join problem. After that, we derive lemmas using the properties of metric spaces. Based on them, we can check filtering conditions very quickly and screen out pairs without full score computation. We finally propose an algorithm to efficiently process the transformed similarity join problem on GPUs. To the best of our knowledge, our framework is the first to sysmatically solve the knowledge graph completion problem. We hope that our initial findings and results will promote and accelerate this area.

The contributions of our paper are as follows:
\begin{itemize}
\item \textbf{Transformation of Knowledge Graph Completion Problem to Similarity Join Problem:}
For a knowledge graph embedding model which is \emph{transformable to a metric space}, we provide a method to transform the knowledge graph completion problem to the similarity join problem by defining $connector_1$, $connector_2$, and $dist$ functions. 
\item \textbf{Derivation of Lemmas from Properties of Metric Spaces:} 
For efficient and effective filtering, we introduce some pivot and derive two lemmas from the properties of metric spaces. Based on the lemmas, we can evaluate the filtering conditions with the time complexity O($N$) and effectively filter out pairs without full computation. Note that $N$ is the number of entities.
\item \textbf{Algorithm to Efficiently Process Knowledge Graph Completion on GPUs:} Based on the lemmas above, we propose an algorithm to process the knowledge graph completion problem on GPUs. To improve the parallelism, we provide a grouping method. In addition, we propose a block-based algorithm to solve the limitation of GPU memory.
\item \textbf{TransE Model Analysis and Extensive Experiments:} To figure out the characteristics of the knowledge graph completion problem, we analyze TransE models trained from several data sets. Also, we conduct extensive experiments with various data sets such as WN18, WN18RR, FB15K and FB15K-237 which are commonly used in the knowledge graph embedding literature.
\end{itemize}

The rest of the paper is organized as follows. In Section \ref{sec:related}, we review the studies in three areas, \emph{knolwedge graph embedding, similarity join problem on CPUs, and similarity join problem on GPUs}. In Section \ref{sec:problem}, we will explain the problem definition and settings in detail. After that, we will describe the detailed parts of our framework in Section \ref{sec:main}. Finally, we will show experimental results and conclude the paper in Sections \ref{sec:exp} and \ref{sec:con}.

\section{Related Work} \label{sec:related}
In this section, we will describe the related work by classifying it into three parts, \textit{knowledge graph embedding}, \textit{similarity join on CPUs} and \textit{similarity join on GPUs}.
  
\subsection{Knowledge Graph Embedding}
Many approaches for knowledge graph embedding have been proposed in the AI literauture \cite{kge:transe,kge:transh,kge:transr,kge:stranse,kge:transf,kge:distmult,kge:complex,kge:simple,kge:conve,kge:acre}. Typically, there are translation-based approaches \cite{kge:transe,kge:transh,kge:transr,kge:stranse,kge:transf}, bilinear approaches \cite{kge:distmult,kge:complex,kge:simple} and deep neural network based approaches \cite{kge:conve,kge:acre}. The translation-based approaches model relations by translation. TransE \cite{kge:transe}, which is a typical translation approach, trains embedding vectors $\textbf{h},\textbf{r}$, and $\textbf{t}$ for each triplet $(h,r,t)$ such that $\textbf{h} + \textbf{r} \simeq \textbf{t}$. However, since TransE can not represent many-to-one/one-to-many/many-to-many relations, several approaches have been developed to overcome it. To solve the issue of many-to-one/one-to-many/many-to-many relations, TransH \cite{kge:transh} projects \textbf{h, t} to the new hyperplane according to $r$ such that $\mathbf{h}_{\bot}+ \mathbf{d}_r \simeq  \mathbf{t}_{\bot}$, where $\mathbf{h}_{\bot} = \mathbf{h} - \mathbf{w}_r^{\top} \mathbf{h} \mathbf{w}_r$ and $\mathbf{t}_{\bot} = \mathbf{t} - \mathbf{w}_r^{\top} \mathbf{t} \mathbf{w}_r$. TransR  \cite{kge:transr}  considers entities and relations in separate spaces. By defining the mapping matrix $\mathbf{M}_r$, TransR attempts to maximize $-|| \mathbf{h}_r + \mathbf{r} - \mathbf{t}_r||^2_2$ where $\mathbf{h}_r = \mathbf{hM}_r$ and $\mathbf{t}_r = \mathbf{tM}_r$. Also, STransE \cite{kge:stranse} combines Structured Embedding \cite{kge:se} and TransE \cite{kge:transe} to improve the accuracy. The score function of STransE is formulated as $ - ||\mathbf{W}_{r,1} \mathbf{h} + \mathbf{r} - \mathbf{W}_{r,2} \mathbf{t} ||_{L_p} $ where $\mathbf{W}_{r,1}$ and $\mathbf{W}_{r,2}$ are relation-specific matrices according to $r$ and $L_p$ is 1 or 2.

Bilinear approaches have a bilinear function such as $<h,r,t>= \Sigma_{i} h_i r_i t_i$. DistMult \cite{kge:distmult} is the simplest bilinear approach. DistMult attempts to maximize $<h,r,t>$. ComplEx \cite{kge:complex} can be considered as a bilinear approach. Entity embedding vectors and relation embedding vectors in ComplEx are defined on a complex space. The score function of ComplEx  is evaluated by the real part of $<\mathbf{w}_r, \mathbf{e}_h, \bar{\mathbf{e}_t}>$, where $
\mathbf{w}_r, \mathbf{e}_h, \mathbf{e}_t \in \mathbb{C}^d$ are embedding vectors corresponding to $r, h, t$, respectively and $\bar{\mathbf{e}_t}$ is the complex conjugate of $\mathbf{e}_t$. SimplE \cite{kge:simple} is a bilinear approach to consider both relations and inverse relations. Therefore, each entity $e$ has two entity embedding vectors $\mathbf{h}_e, \mathbf{t}_e$ and each relation $r$ has two relation embedding vectors $\mathbf{r}, \mathbf{r}_{inverse}$. Given a triple $(e_1, r, e_2)$, SimpleE attempts to maximize $\frac{1}{2}(<\mathbf{h}_{e_1},\mathbf{r},\mathbf{t}_{e_2}> + <\mathbf{h}_{e_2},\mathbf{r}_{inverse},\mathbf{t}_{e_1}>$).

Deep neural network based approaches use neural network architectures such as CNNs. ConvE \cite{kge:conve} uses convolutional layers and fully-connected layers to train knowledge graph embedding vectors whereas AcrE \cite{kge:acre} builds neural networks with standard convolutions, atrous convolutions and residual networks. Also, AcrE proposes two types of architectures, serial architecture and parallel architecture. In both ConvE and ArcE, a head embedding vector and a relation embedding vector are merged and are reshaped into two dimensional data. The reshaped data are used as input of the neural architecture. 

In addition, to improve the performance of the existing knowledge graph embedding models in terms of  the accuracy (e.g., Hit Rate, Mean Rank, Mean Reciprocal Rank), various methods such as ensemble and  new negative sampling techniques are proposed  \cite{kge:ensemble,kge:gan1,kge:gan2}. Krompa\ss~et al. \cite{kge:ensemble} propose a simple ensemble method with RESCAL, TransE, and neural network based model (called \emph{mwNN} \cite{kge:mwNN}) and experimentally show the ensemble method improves the accuracy of the original knowledge graph embedding models with data sets DBPedia-Music, Freebase-15K and YAGOc-195K. Moreover, to boost the accuracy of the existing knowledge graph embedding models, new negative sampling techniques in knowledge graph embedding are devised \cite{kge:gan1,kge:gan2}. They use GAN (Generative Adversarial Network) \cite{gan} to generate better negative samples.
 
\subsection{Similarity Join on CPUs}

There are plenty of approaches dealing with similarity join problems on CPUs.
For low dimensional data (typically 2 or 3 dimensions), a similarity join problem is strongly related to a spatial join problem. The spatial join problem for spatial object data sets $A$ and $B$ can  be defined as follows \cite{spatial:multistep,spatial:s3j}:
\begin{equation} \label{eq:simdef}
\begin{split}
A  \bowtie_{\theta}   B  =  \{ (a,b) | \theta(a,b) = True   \text{ for } a\in A, b \in B\}
\end{split}
\end{equation}
where $\theta$ is a spatial predicate. For example, we can use \emph{overlap(a,b)} or \emph{near(a,b)} as a spatial predicate. \emph{overlap(a,b)} returns True if spatial objects $a$ and $b$ overlap while \emph{near(a,b)} returns True if $a$ and $b$ are close within threshold $\epsilon$. Partition-based Spatial Merge Join \cite{spatial:partition} stores spatial data to the corresponding partitions. If some spatial object overlaps with multiple partitions, it is replicated and stored in multiple partitions for computing correct answers. However, Size Separation Spatial Join \cite{spatial:s3j}, called S3J, uses level files to avoid spatial object replication. The level files consist of the hierarchical structure from the high resolution file to the low resolution file. Each object is stored on the appropriate level file in order not to overlap with multiple partitions.

The similarity join problem can be formulated in the same way as Equation \eqref{eq:simdef} if we let $\theta(a,b)$ be $dist(a,b) \le \epsilon$ and $A$ and $B$ be $d$-dimensional points. To process the similarity join problem on CPUs, various techchniques have been developed. One of the classical approaches to handle the problem is to use tree-based approaches such as R-tree \cite{spatial:rtree}, $\epsilon$-KDB tree \cite{sim:ekdb}. Basically, after construction of the tree-based indices on the fly, join operations are performed using the constructed indices. For example, $\epsilon$-KDB tree \cite{sim:ekdb} first builds a tree-index with data points  
by splitting the space of a node into $\lfloor \frac{1}{\epsilon} \rfloor$ subspaces. When a node is split at level i, only dimension i is partitioned. After that, the join algorithm is performed using the  $\epsilon$-KDB tree.

As another approach, we can use a grid-based approach \cite{sim:ego,sim:ego2}. In the grid-based approach, the size of each grid is set to $\epsilon$ in order to efficiently filter out unnecessary grids. In \cite{sim:ego}, Epsilon Grid Order (EGO) is devised to process the disk-based similarity join for a large number of high dimensional data. EGO consists of $\epsilon$ sized grids conceptually and each grid is sorted according to the epsilon grid order. The sorted grids are grouped sequentially with respect to the given number of points and stored into IO units. After that, join operations have to be performed on pairs of IO units. Also, in \cite{sim:ego}, an effective scheduling algorithm is proposed to reduce the number of page loadings. Kalashnikov et al. \cite{sim:ego2} propose two grid-based approaches, \emph{Grid-join} and \emph{EGO*-join}. Consider the similarity join with data sets $A$ and $B$. In Grid-join, each point of $B$ is transformed into the circle with $\epsilon$ radius and then a grid index with the circles is constructed. Each cell in the grid index has full and part lists. The full list stores the elements whose circle includes the grid completely whereas the part list stores the elements whose circle intersects with the grid. To process the join problem based on the grid index, Grid-join first finds the corresponding grid to each point of $A$. For all the elements of the full list in the grid, Grid-join appends them to the result buffer without verification. For the elements of the part list in the grid, Grid-join checks if they are a real answer or not and appends the verified elements to the result buffer. Note that the full list concept is applied to only two dimensional case. In addition, Kalashnikov et al. \cite{sim:ego2} propose EGO*-join based on EGO-join \cite{sim:ego}. They devise advanced non-joinable conditions compared to EGO-join.

The previous approaches above assume the distance is the $L_p$ norm and they cannot be applied to other distance functions. Jacox et al. \cite{sim:quickjoin} propose Quickjoin to efficiently process the similarity join problem whose distance function is a metric function. In addition, Silva et al. \cite{sim:dboperator} tackle the similarity join problem on a metric space as a database operator. They propose a similarity join algorithm based on QuickJoin, called \emph{DBSimJoin}, and present the parser, the planner, and the executor for DBSimJoin.

\subsection{Similarity Join on GPUs}

Recently, some approaches attempt to solve the similarity join problem in different environments such as GPUs or distributed systems.  Since a distributed environment such as Hadoop is not our focus, we will discuss the related work on a GPU environment. B\"ohm et al. \cite{sim2:index} show that the index-based similarity join algorithm on GPUs has significant improvement compared to the similarity join algorithm without index on CPUs. They use a very simple index like a directory page and each node of the index has the same number of child nodes. Even though they provide a method to process the similarity join problem on GPUs, they do not consider the high dimensional data. They conduct the experiments for only 8 dimensional data. 

Gowanlock et al. \cite{sim2:grid} handle the large-scale similarity self-join problem using GPUs. Basically, they use a grid-based approach. Given $\epsilon$, they build an $\epsilon$ sized grid for each dimension. A data point $p=(x_1, x_2, \cdots\, x_d)$ is stored in grid $g$ conceptually such that $x_i  \in [~g[i].left,~g[i].right~)$, where $g[i].left$ and  $g[i].right$ are the boundary coordinates of grid $g$ in the i-th dimension. To reduce the space consumption, they store only non-empty cells. For data search in grids, four data structures, $B, G, A$ and $M$ are kept  \cite{sim2:grid}. $B$ is a grid cell lookup and has a pointer to the cell information whereas $G$ is a grid-point mapping and has the point range information. $A$ is a point lookup array and lists all the data points according to the cell ids. $M$ is a masking array to check if the intervals in each dimension have no data point.  $M$ is used to reduce the query region. Because the approach in \cite{sim2:grid} is a grid-based approach, it has the intrinsic drawback of grid-based approaches. That is, the number of grids increases exponentially as the dimensionality gets higher. 

Lieberman et al. \cite{sim2:zorder} also solve the similarity join problem using GPUs. They use a space-filling curve called Z-order because Z-order has a good property as follows:

\begin{equation}
a[i] < b[i] \text{~for all } i \in d \Rightarrow Z(a)<Z(b). \label{eq2}
\end{equation}
 
Using the property above and a set of Z-lists, they filter out unnecessary data points. However, the cost of translating high-dimensional data to the set of Z-lists is high. Note that the time complexity for that is  $O(d^2~\frac{n}{k}~log^2 n)$ \cite{sim2:zorder}, where $n$ is the number of d-dimensional objects and $k$ is the number of threads performed in parallel. If $d$ is high, the term $d^2$ is not a constant term. Also, they do not consider the case of the error threshold $\epsilon > 1$.

Gallet et al. \cite{sim2:load} deal with the issue of the workload imbalance when evaluating the similarity join problem on GPUs. If we assign each thread of a GPU to the job by point p (i.e., each thread has its job of evaluating $|p - q|$ for all $q \in D$), the workload imbalance will occur because in most cases, data are skewed. To remove the workload imbalance, they develop many techniques such as fine-grained job assigning, linear ID mapping, sorting by the amount of workload.

\section{Problem Statement} \label{sec:problem}
In this section, we will describe problem definition and assumptions.

\subsection{Knowledge Graph Completion Problem} \label{sec:problem1}
Knowledge graph completion is to add missing links to the existing knowledge to solve the sparseness of the knowledge graph. To add a new link using a knowledge graph embedding model, we can perform triple classification for all possible cases. That is, if $score(\mathbf{h}_i, \mathbf{r}_j, \mathbf{t}_k) \ge \epsilon$ , the link will be added to the knowledge graph. To use the analogous notation to the similarity join problem, we employ $dist_3(\mathbf{h}_i, \mathbf{r}_j, \mathbf{t}_k) = -score(\mathbf{h}_i, \mathbf{r}_j, \mathbf{t}_k) $ and add the link if $dist_3(\mathbf{h}_i, \mathbf{r}_j, \mathbf{t}_k) \le \epsilon$. The knowledge graph completion problem using knowledge graph embedding vectors is formulated as follows: 

\begin{definition}(Knowledge Graph Completion Problem) \label{def:problem}
Given a knowledge graph embedding model with entity embedding vectors $E$, relation embedding vectors $R$ and function $score$: $E \times R \times E \rightarrow \mathbb{R}$, find all the pairs $(\mathbf{e}_i, \mathbf{r}_j, \mathbf{e}_k)$ such that $dist_3(\mathbf{e}_i, \mathbf{r}_j, \mathbf{e}_k) \le \epsilon$ where $\mathbf{e}_i, \mathbf{e}_k \in E$, $\mathbf{r}_j \in R$ and $dist_3(\mathbf{x, y, z}) = -score(\mathbf{x, y, z})$.
\end{definition}

\begin{figure}
\centering
\includegraphics[width=4in,page=2,trim=0 8in 0 0,clip]{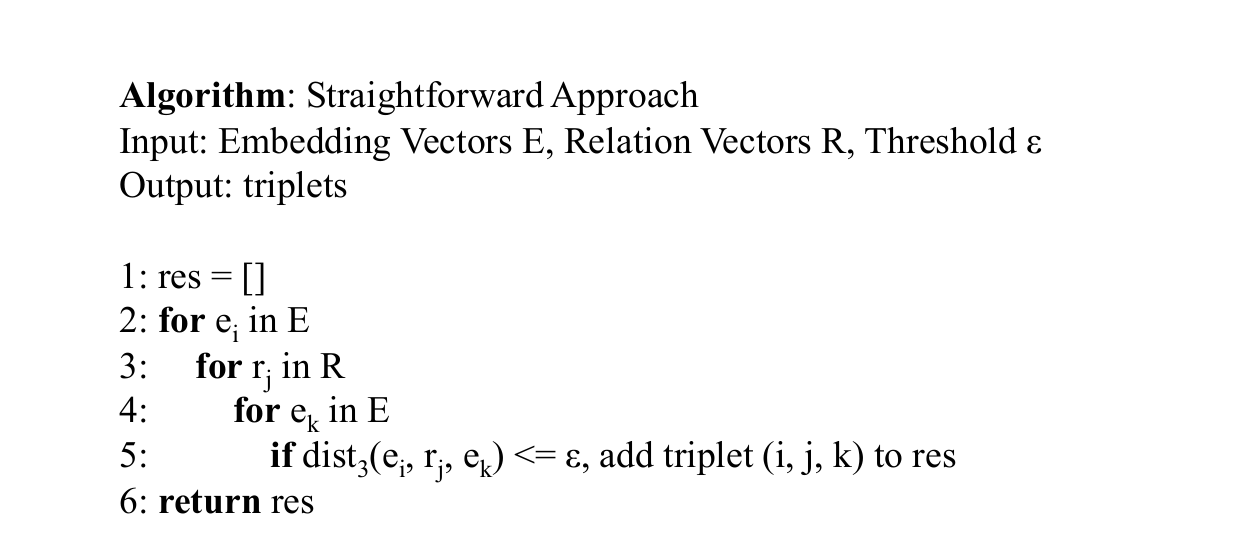}
\caption{Straightforward Algorithm for Knowledge Graph Completion}
\label{fig:naive}
\end{figure}

The knowledge graph completion problem by embedding vectors can be solved with a very simple algorithm as shown in Figure \ref{fig:naive}. For each triplet (i.e., head embedding vector $\mathbf{e}_i$, relation embedding vector $\mathbf{r}_j$ and tail embedding vector $\mathbf{e}_k$), the straightforward algorithm checks $dist_3(\mathbf{e}_i, \mathbf{r}_j, \mathbf{e}_k)$ (See Line 5). If the similarity value is less than or equals to $\epsilon$, the triplet will be appended to the result. However, the cost of executing the straightforward algorithm in Figure \ref{fig:naive} is prohibitive in terms of the running time because many knowledge graphs have more than million nodes. For instance, in the case of one million nodes with 1,000  relation types, we have to evaluate $10^{15}$ similarity functions.

\subsection{Focus of TransE and Its Variants}  \label{sec:problem3}
We aim at evaluating \emph{Knowledge Graph Completion Problem} in Definition \ref{def:problem} as quickly as possible. However, there are many different types of score functions such as TransE \cite{kge:transe}, DistMult \cite{kge:distmult}, ComplEx \cite{kge:complex}. Realistically, it would be not easy to find a common methodology to evaluate scores fast. Originally, we attempted to develop a general framework to cope with most knowledge graph embedding models. However, due to variability and complexity of score functions, we failed to find a common methodology to speed up the time of evaluating the completion problem in any knowledge graph embedding models. After careful consideration, we decided to focus on one typical knowledge graph embedding model and chose TransE \cite{kge:transe} because of the following reasons.

\begin{itemize}
\item \textbf{Good Accuracy:} Although TransE is a simple approach, its performance is comparable with the recent approaches in terms of accuracy. According to the experiments of \cite{kge:performance2,kge:performance3}, TransE has a comparable accuracy with the recent approaches if their parameters are properly tuned.

\item \textbf{Many Variants:} TransE has many variants such as STransE \cite{kge:stranse}, TransF \cite{kge:transf}, TransH \cite{kge:transh}, TransR \cite{kge:transr}. This is  because the concept of TransE is clear and its implementation is easy. 
 
\item \textbf{Simplicity:} Since TransE has a simple structure, it would be very useful as a pilot model for our framework. For example, it is easy to apply various mathematical concepts because of simplicity.
\end{itemize}

Therefore, we concentrate on a technique of speeding up the running time of evaluating the knowledge graph completion problem with TransE embedding vectors. To extend the range of applications, we introduce \emph{transformable to a metric space} and propose a framework to apply to not only TransE but also models which are \emph{transformable to a metric space}.

To the best of our knowledge, our framework is the first to systematically solve the knowledge graph completion problem with embedding vectors in terms of the running time. We hope that our framework will promote this area and a wide range of useful techniques will be developed by integrating the techniques in the DB literature.

\subsection{Assumtions in Knowledge Graph Completion Problem}  \label{sec:problem2}
Although the problem we would like to solve in this paper is formulated in Definition \ref{def:problem}, we assume the following:

\begin{itemize}
\item \textbf{High Dimensional Data:} Knowledge graph embedding vectors usually have hundreds of dimensions. Figure \ref{fig:modelAcc}-(b) shows the characteristics of trained embedding models by using LibKGE \cite{kge:libkge} which is a framework to easily build and evaluate a knowledge graph embedding model. The dimensions of the models are 128, 512, 512, and 128, respectively. \footnote{The detailed description of data sets and trained models will be explained in Section \ref{sec:exp}. } Due to the curse of dimensionality, the dimension size significantly impacts on methodologies. In this paper, we focus on developing techniques for high dimensional data.

\item $|E| >> |R|$:  In knowledge graphs, the number of entities ($|E|$) is much greater than the number of relation types ($|R|$). For example, in the WN18RR data set, the number of relation types is 11 whereas that of entities is 40,943 as shown in Figure \ref{fig:modelAcc}-(a).
\item $\mathbf{\epsilon}$ \textbf{Size:} Many approaches to process the similarity join problem do not work well when $\epsilon$ is big compared to the range of each dimension (especially, $\epsilon > 1$ in space $[0,1]^d$). However, according to our analysis, big $\epsilon$ should be supported in the knowledge graph completion problem. As shown in Figure \ref{fig:top1}, for four pretrained TransE models using LibKGE \cite{kge:libkge} \footnote{We can download them from the site https://github.com/uma-pi1/kge.}, we evaluated the top-1 score (i.e., $min_{i,j,k} (\mathbf{h_i} + \mathbf{r_j} - \mathbf{t_k}$)) and the top-1 score except self-edges (i.e., $min_{i,j,k~\text{and}~(i \neq k)} (\mathbf{h_i} + \mathbf{r_j} - \mathbf{t_k}$)). In FB15K, the top-1 score (2.2475) and the top-1 score except self-edges (4.1478) are greater than the max range size 1.9864. It means that we have to set $\epsilon$ to greater than 1 in space $[0, 1]^d$.
\end{itemize}

\begin{figure}
\centering
\includegraphics[width=4in,page=4,trim=0 5.8in 0 0,clip]{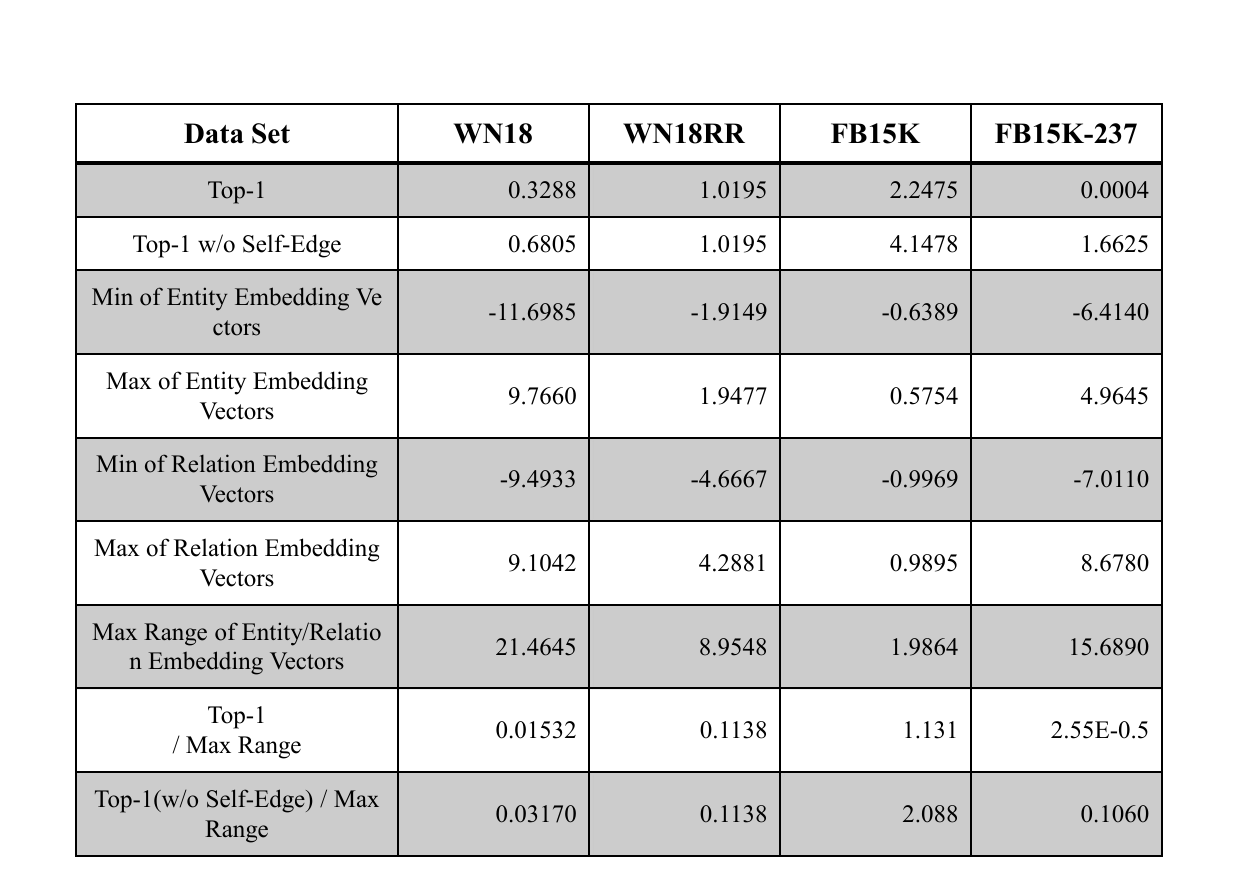}
\caption{Top-1 Values and Ranges}
\label{fig:top1}
\end{figure}

\section{Efficient Knowledge Completion Framework on GPUs} \label{sec:main}
In this section, we will provide the detailed explanation of our framework. We will first describe the overall approach in Section \ref{sec:main1} and then introduce a metric space in Section \ref{sec:main2}, which is necessary to understand our framework. After that, we will explain the details of our framework in Sections \ref{sec:main3}, \ref{sec:main4},  \ref{sec:main5}, \ref{sec:main6} and \ref{sec:main7}.

\subsection{Overall Approach}   \label{sec:main1}

\begin{figure}
\centering
\includegraphics[width=4.5in,page=5,trim=0 8in 0 0,clip]{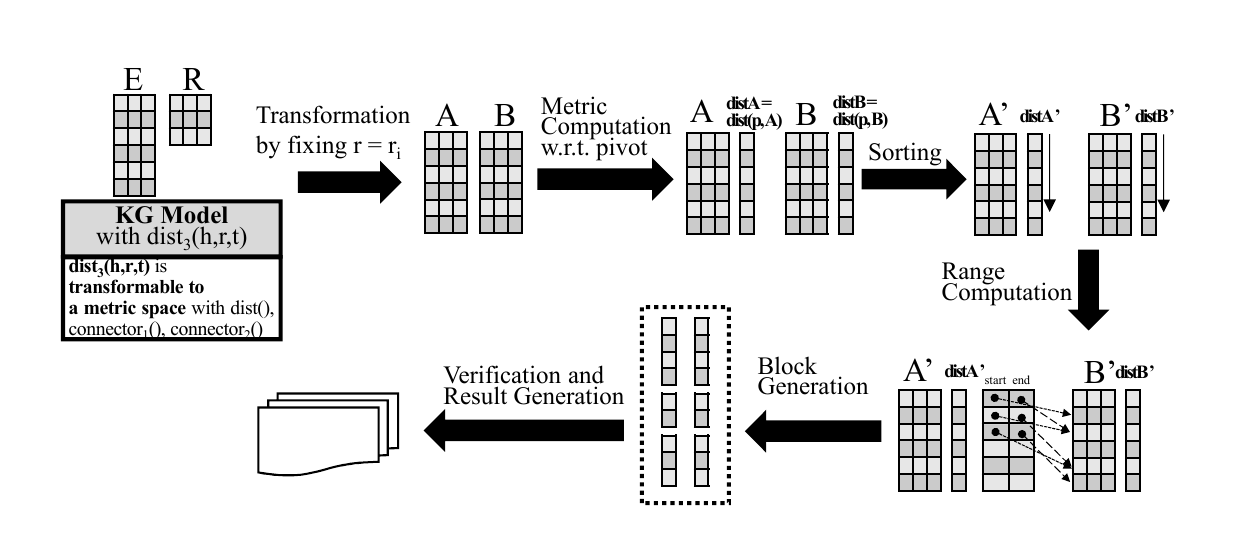}
\caption{Overall Approach}
\label{fig:overall}
\end{figure}

We first describe the overall approach as shown in Figure \ref{fig:overall}. Although we focus on TransE, we attempt to extend our framework in order to be applied to a wide range of knowledge graph embedding models. To do that, we introduce the definition \emph{transformable to a metric space}, which will be explained later. If the similarity function $dist_3$ of knowledge graph embedding models ($dist_3(\mathbf{h, r, t}) = - score(\mathbf{h, r, t})$) is \emph{transformable to a metric space}, there exist $connector_1: \mathbb{R}^d \times \mathbb{R}^d  \rightarrow \mathbb{R}^d$, $connector_2: \mathbb{R}^d \times \mathbb{R}^d  \rightarrow \mathbb{R}^d $ and $dist: \mathbb{R}^d \times \mathbb{R}^d  \rightarrow \mathbb{R} $ functions, where $dist$ is a metric space. 

To apply our framework to a knowledge graph embedding model which is \emph{transformable to a metric space}, we first transform triple data into binary data for each $r=r_i$. For the transformation, $connector_1$ and $connector_2$ functions are used. After that, we get two vector sets, $A, B$. Now we solve the similarity join problem $dist(\mathbf{a,b}) \le \epsilon$ for $\mathbf{a} \in A$ and $\mathbf{b} \in B$ instead of $score(\mathbf{h,r,t}) \ge \epsilon$ for $\mathbf{h,t} \in E$ and $\mathbf{r} \in R$. To efficiently process the similarity join problem, we derive some formulas using the properties of a metric space. In order to apply our formulas, we choose one pivot \textbf{p} and evaluate similarity values between $\textbf{a} \in A$ (or $\textbf{b} \in B$) and pivot \textbf{p}, respectively. They are denoted by \textbf{distA} and \textbf{distB}. We then sort \textbf{distA} and \textbf{distB}, respectively, and rearrange $A$ and $B$ accordingly. Sorted \textbf{distA} and sorted \textbf{distB} are denoted by \textbf{distA'} and \textbf{distB'}, respectively whereas sorted $A$ and sorted $B$ are by $A'$ and $B'$, respectively. The sorting is not heavy since an element of \textbf{distA} and \textbf{distB} is one-dimensional data and the list sizes are $|A|$ and $|B|$, respectively. From the sorted lists (\textbf{distA'} and \textbf{distB'}), we compute the valid range (i.e., \emph{start, end}) for each element \textbf{a} in $A'$ using Lemmas \ref{lemma:lem1} and \ref{lemma:lem2} (This step is processed using Algorithm \ref{fig:algo2}). 

This process seems costly but is performed in $O(max(|A|, |B|))$ because of Lemmas \ref{lemma:lem1} and \ref{lemma:lem2}. Using valid ranges, we create a block which can be processed in parallel on a GPU. Finally, we compute the similarity values (or score values) for all the pairs in the block, verify if the results are valid and return the valid results. We iterate the above process for each $\mathbf{r_i}$ in $R$.

\subsection{Metric Spaces}  \label{sec:main2}
Although we focus on TransE, we extend our framework by introducing a metric space. The metric space is defined as follows: \cite{wiki:metric}

\begin{definition}
Given a set $S$ and a distance function $dist: S \times S \rightarrow \mathbb{R}$, if $dist$ satisfies the following four conditions, we call ($S$, $dist$) a metric space.
\begin{itemize}
\item $dist(a,a) = 0$ for any $a \in S$
\item $dist(a,b) > 0$ for any $a,b (a \neq b) \in S$
\item $dist(a,b) = dist(b,a)$ for any $a,b \in S$ (Symmetry)
\item $dist(a,b) \le dist(a,c) + dist(c,b)$ for any $a,b,c \in S$ (Triangle Inequality)
\end{itemize}
\end{definition}

We derive formulas on a metric space. However, a score function has three input parameters whereas a metric function has only two input parameters. Thus, we transform the knowledge graph completion problem to the similarity join problem.

\subsection{Transformation to Similarity Join Problem} \label{sec:main3}

A score function in the knowledge graph completion problem has three input parameters whereas a similarity function in the similarity join problem has two input parameters. To transform a ternary function into a binary function, we  define
$connector_1: \mathbb{R}^d \times \mathbb{R}^d  \rightarrow \mathbb{R}^d$, $connector_2: \mathbb{R}^d \times \mathbb{R}^d  \rightarrow \mathbb{R}^d $ and $dist: \mathbb{R}^d \times \mathbb{R}^d  \rightarrow \mathbb{R} $ functions such that

\begin{equation}
dist_3(\mathbf{h, r, t}) = dist(connector_1(\mathbf{h, r}), connector_2(\mathbf{t, r}))
\end{equation}
where $dist$ is a metric function. If there exist $dist$, $connector_1$, and $connector_2$ for $dist_3(\mathbf{h,r,t})$,  we say that $dist_3(\mathbf{h, r, t})$ is \emph{transformable to a metric space} with $dist$, $connector_1$, and $connector_2$. After letting $\mathbf{a}=connector_1(\mathbf{h, r})$ and $\mathbf{b}=connector_2 (\mathbf{t, r})$, the knowledge graph completion problem is formulated as follows:

\begin{equation}
\begin{split}
dist_3(\mathbf{h, r, t}) & = dist(connector_1(\mathbf{h, r}),  connector_2(\mathbf{t, r})) \\ 
                        &= dist (\mathbf{a, b}) \le \epsilon
\end{split}
\end{equation}

In TransE \cite{kge:transe}, $dist_3(\mathbf{h, r, t} ) = || \mathbf{h + r - t}||_{L_p}$. If we let $dist(\mathbf{a,b}) = || \mathbf{a - b}||_{L_p}$, $connector_1(\mathbf{a,b})=  \mathbf{a+b}$ and $connector_2(\mathbf{a,b}) = \mathbf{a}$, $dist_3(\mathbf{h, r, t}) = dist( connector_1(\mathbf{h,r}), connector_2(\mathbf{t, r}))$. Thus, TransE is \emph{transformable to a metric space} with $dist$, $connector_1$ and $connector_2$ above. In SE \cite{kge:se}, $dist_3(\mathbf{h, r, t})= || \mathbf{W}_r^{lhs} \mathbf{h} - \mathbf{W}_r^{rhs} \mathbf{t} ||_1$, where $\mathbf{W}_r^{lhs}, \mathbf{W}_r^{rhs}$ are $d \times d$ relation specific matrices. If we let  $dist(\mathbf{a, b})= || \mathbf{a - b}||_1$, $connector_1(\mathbf{a,b}) =  \mathbf{W}_b^{lhs} \mathbf{a}$ and $connector_2(\mathbf{a,b}) = \mathbf{W}_b^{rhs} \mathbf{a}$, SE is also \emph{transformable to a metric space}. Therefore, we can apply our framework to TransE and SE.

Now we handle $dist(\mathbf{a,b})$ instead of  $dist_3(\mathbf{h,r,t})$ for a knowledge graph embedding model which is \emph{transformable to a metric space}. That is, we consider the similarity join problem instead of the knowledge graph completion problem.

\subsection{Lemmas for Our Framework} \label{sec:main4}
To effectively filter out element pairs, we derive the following lemmas using properties of metric functions.  
Based on these lemmas, we will develp a fast completion algorithm on a GPU in Section \ref{sec:main5}.

\begin{lemma} \label{lemma:lem1}
 If $(\mathbb{R}^d, dist)$ is a metric space and some pivot $\mathbf{p}$ is given,
\text{for any} $\mathbf{a, b}$
\begin{equation} \label{eq:lem1}
\begin{split}
dist(\mathbf{a, b})  \ge | dist(\mathbf{p, a}) - dist(\mathbf{p, b}) | 
\end{split}
\end{equation}
\end{lemma}

\begin{proof}
To prove $dist (\mathbf{a, b})  \ge | dist(\mathbf{p, a}) - dist(\mathbf{p, b}) |$, we use two properties of a metric space, Symmetry and Triangle Inqueality. To avoid confusion, we use $\mathbf{x,y,z}$ instead of $\mathbf{a,b,c}$. Then, for any $\mathbf{x,y,z}$,
\begin{equation}
\begin{split}
dist (\mathbf{x, y})  \le dist(\mathbf{x, z}) + dist(\mathbf{z, y})&  \\
\text{(Triangle Inequality)}&
\end{split}
\end{equation}
If we let $\mathbf{x=a}$, $\mathbf{y=p}$, and $\mathbf{z=b}$, 
\begin{equation} \label{eq:triangle1}
\begin{split}
dist (\mathbf{a, p})  &\le dist(\mathbf{a, b}) + dist(\mathbf{b, p}) \\
\Rightarrow dist (\mathbf{a, b})  &\ge dist(\mathbf{a, p}) - dist(\mathbf{b, p})  \\
\Rightarrow dist (\mathbf{a, b})  &\ge dist(\mathbf{p, a}) - dist(\mathbf{p, b})  ~~~~(Symmetry)\\
\end{split}
\end{equation}
If we let $\mathbf{x=b}$, $\mathbf{y=p}$, and $\mathbf{z=a}$, 
\begin{equation}  \label{eq:triangle2}
\begin{split}
dist (\mathbf{b, p})  &\le dist(\mathbf{b, a}) + dist(\mathbf{a, p}) \\
\Rightarrow dist (\mathbf{b, a})  &\ge dist(\mathbf{b, p}) - dist(\mathbf{a, p})  \\
\Rightarrow dist (\mathbf{a, b})  &\ge dist(\mathbf{p,b}) - dist(\mathbf{p,a})  ~~~~(Symmetry)\\
\Rightarrow dist (\mathbf{a, b})  &\ge - (dist(\mathbf{p,a}) - dist(\mathbf{p,b}))  \\
\end{split}
\end{equation}
From inequalities  \eqref{eq:triangle1} and  \eqref{eq:triangle2}, the following formula holds for any $\mathbf{a,b,c}$
\begin{equation}
\begin{split} 
dist (\mathbf{a, b})  \ge  max( &dist(\mathbf{p,a} )- dist(\mathbf{b, p}),  \\
&-(dist(\mathbf{p,a}) - dist(\mathbf{p,b})) ) \\
\end{split}
\end{equation}
Therefore, 
\begin{equation}
\begin{split} 
dist (\mathbf{a, b})  \ge  | dist(\mathbf{p,a})- dist(\mathbf{b, p})|
\end{split}
\end{equation}
\end{proof}
\hfill$$ \square $$

Actually, Lemma \ref{lemma:lem1} and similar inequalities derived from the properties of a metric space have been used commonly. For instance, based on the same inequality as \eqref{eq:lem1}, Chen et al. \cite{sim:pivot} propose methods to process similarity search and similarity join with prebuilt indices. To process similarity search and similarity join efficiently, they choose $k$ pivot sets and transform a data object $o$ to a k-dimensional vector space ($\phi(o)=<dist(o, p_1), dist(o, p_2), \cdots, dist(o, p_k)>$). After that, they process similarity search and similarity join in the transformed vector space. This approach is valid because the $L_{\infty}$ distance in the transformed vector space is the lower bound of the distance in the metric. Note that $dist(o_i, o_j) \ge ||\phi(o_i) - \phi(o_j)||_{\infty}= max\{ |d(o_i, p_t) - d(o_j, p_t)| | p_t \in P  \}$ from $d(o_i, o_j) \ge |d(o_i, p) - d(o_j,p)|$, where $P$ is a set of pivots \cite{sim:pivot}. However, we use Lemma \ref{lemma:lem1} with Lemma \ref{lemma:lem2} to get the range of candidate elemets to be evaluated from the sorted list according to the pivot. In addition, Chen et al. focus on similarity search and similarity join with prebuilt indices. That is, they attempt to reduce not the total time including index construction but the query time. Moreover, different from our definition, the similarity join problem in \cite{sim:pivot} is formulated as $SimJoin(Q,O) = \{(q,o) | dist(q,o) \le \epsilon
~for~q~ \in Q, o \in O  \}$, where $Q$ is a query set and $O$ is objects.

Let us analyze the inequality in Lemma \ref{lemma:lem1} carefully. For easy understanding, assume that we have $A=\{\mathbf{a_1, a_2, a_3}\}$ and $B=\{\mathbf{b_1, b_2, b_3}\}$ and we would like to check if $dist (\mathbf{a, b}) > \epsilon$ without direct computation for $\mathbf{a} \in A$ and $\mathbf{b} \in B$ using Lemma \ref{lemma:lem1}.

\begin{equation}
\begin{split}
dist (\mathbf{a_1, b_1})  &\ge | dist(\mathbf{p, a_1}) - dist(\mathbf{p, b_1}) | \\
dist (\mathbf{a_1, b_2})  &\ge | dist(\mathbf{p, a_1}) - dist(\mathbf{p, b_2}) | \\
dist (\mathbf{a_1, b_3})  &\ge | dist(\mathbf{p, a_1}) - dist(\mathbf{p, b_3}) | \\
dist (\mathbf{a_2, b_1})  &\ge | dist(\mathbf{p, a_2}) - dist(\mathbf{p, b_1}) | \\
dist (\mathbf{a_2, b_2})  &\ge | dist(\mathbf{p, a_2}) - dist(\mathbf{p, b_2}) | \\
dist (\mathbf{a_2, b_3})  &\ge | dist(\mathbf{p, a_2}) - dist(\mathbf{p, b_3}) | \\
dist (\mathbf{a_3, b_1})  &\ge | dist(\mathbf{p, a_3}) - dist(\mathbf{p, b_1}) | \\
dist (\mathbf{a_3, b_2})  &\ge | dist(\mathbf{p, a_3}) - dist(\mathbf{p, b_2}) | \\
dist (\mathbf{a_3, b_3})  &\ge | dist(\mathbf{p, a_3}) - dist(\mathbf{p, b_3}) | \\
\end{split}
\end{equation}

We have 9 unique $dist(\mathbf{a_i,b_j})$ expressions in the left-hand side. That is, we have $|A| \times |B|$ $dist(\mathbf{a_i,b_j})$ unique expressions. However, in the right hand side, we have just 6 unique expressions, $dist(\mathbf{p, a_i})$ and $dist(\mathbf{p, b_j})$. In other words, we have $|A| + |B|$ unique expressions in the right hand side. It means that we can check  $|A| \times |B|$ inequalities (i.e., $dist (\mathbf{a_i, b_j}) >  \epsilon$) with $|A|+|B|$ values.

We now consider the sorted lists for $dist(\mathbf{p,a_i})$ and $dist(\mathbf{p,b_j})$, respectively. Given  $A=\{\mathbf{a_1, a_2, \cdots, a_m} \}$ and $B=\{\mathbf{b_1, b_2, \cdots, b_n}\}$, we let $distA = [ dist(\mathbf{p, a_1}), dist(\mathbf{p, a_2}), \cdots dist(\mathbf{p, a_m}) ]$ and $distB = [ dist(\mathbf{p, b_1}), $ $dist
(\mathbf{p, b_2}), \cdots dist(\mathbf{p, b_n}) ]$.  Also, we denote two sorted lists of $distA$ and $distB$ by 
 $distA' = [ dist(\mathbf{p, a_{\pi _A(1)}}), dist(\mathbf{p, a_{\pi _A(2)}}), \cdots dist(\mathbf{p, a_{\pi _A(m)}}) ]$ and  
 $distB' = [ dist(\mathbf{p, b_{\pi _B(1)}}), dist(\mathbf{p, b_{\pi _B(2)}}), \cdots dist(\mathbf{p, b_{\pi _B(n)}}) ]$, respectively. Then, we can derive the following lemma.

\begin{lemma} \label{lemma:lem2}
If we define $s_i$ and $e_i$ as follows:

\begin{equation} \label{eq:condition}
\begin{split}
s_i&=min_{  k \in   \{ j  ~| ~  | dist(\mathbf{p, a_{\pi_A(i)}}) - dist(\mathbf{p, b_{\pi_B(j)}})| \le \epsilon  \} }    k \\
e_i&=max_{k \in   \{  j~|~ | dist(\mathbf{p, a_{\pi_A(i)}}) - dist(\mathbf{p, b_{\pi_B(j)}}) | \le \epsilon  \} }  k \\
\end{split}
\end{equation}

then, for any $i \in \{1,2, \cdots, m\}$ 
\begin{equation} \label{eq:se}
\begin{split}
| dist(\mathbf{p, a_{\pi_A(i)}}) - dist(\mathbf{p, b_{\pi_B(k)}})| \le \epsilon  & \text{~if~ }  k \in [s_i,e_i] \\
| dist(\mathbf{p, a_{\pi_A(i)}}) - dist(\mathbf{p, b_{\pi_B(k)}})| > \epsilon & \text{~otherwise~ }  \\
\end{split} 
\end{equation}

Also, 
\begin{equation}  \label{eq:se2}
s_i \le s_{i+1} ~and ~e_i \le e_{i+1}
\end{equation}
\end{lemma}

\begin{figure}
\centering
\includegraphics[width=5.5in,page=6,trim=0 6.7in 0 0,clip]{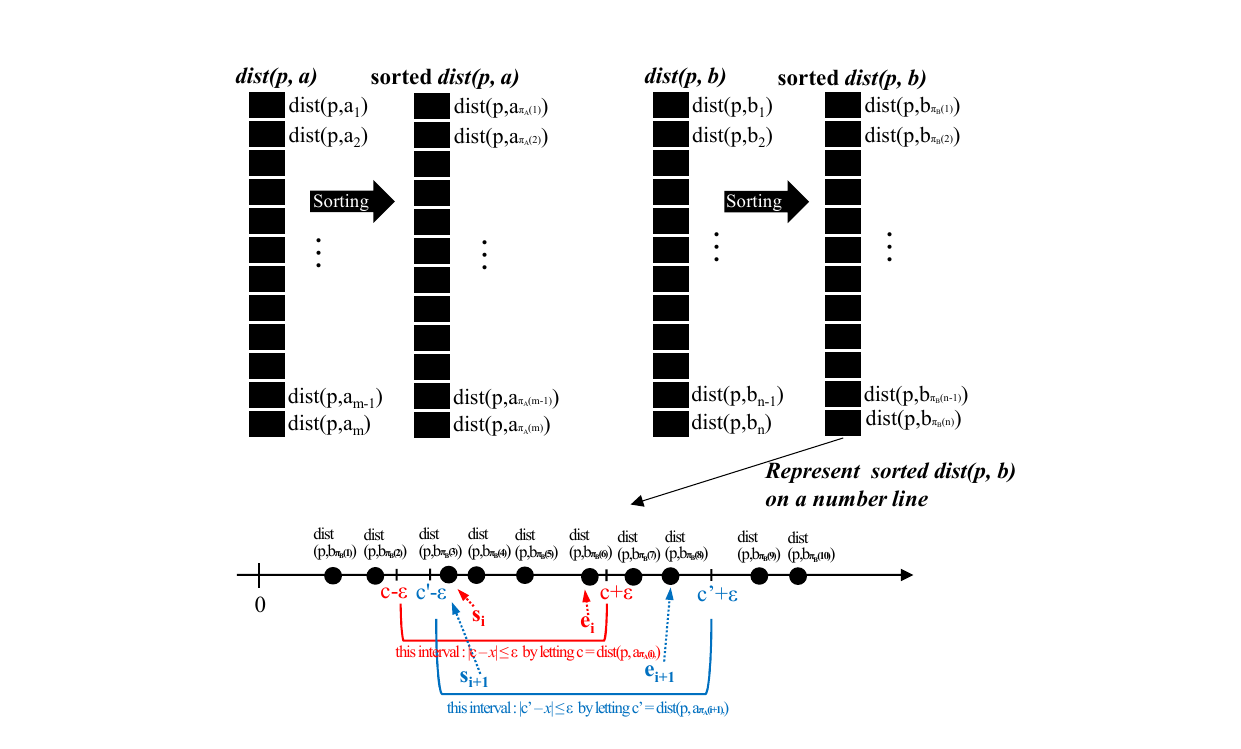}
\caption{Proof}
\label{fig:proof}
\end{figure}

\begin{proof}
See  the inequality $| dist(\mathbf{p}, \mathbf{a}_{\pi_A(i)}) - dist(\mathbf{p}, \mathbf{b}_{\pi_B(j)}) | \le \epsilon$ in the $min$ or $max$ expression. Since the $min$ expression (or $max$) is tied to only $j$, we can consider $dist(\mathbf{p}, \mathbf{a}_{\pi_A(i)})$ as a constant. If we let $c$ be  $dist(\mathbf{p}, \mathbf{a}_{\pi_A(i)})$,  

\begin{equation}
\begin{split} 
& | c - dist(\mathbf{p}, \mathbf{b}_{\pi_B(j)}) | \le \epsilon \\
\Rightarrow & c - \epsilon\le dist(\mathbf{p}, \mathbf{b}_{\pi_B(j)}) \le  c +  \epsilon \\
\end{split}
\end{equation}

The inequality above can be represented on the number line in Figure \ref{fig:proof}. $s_i$ is the first point in the interval by  $| c - dist(\mathbf{p, b_{\pi_B(j)}}) | \le \epsilon$ and $e_i$ is the last point in the interval. Thus, it is trivial that 
\begin{equation}
\begin{split}
| dist(\mathbf{p}, \mathbf{a}_{\pi_A(i)}) - dist(\mathbf{p}, \mathbf{b}_{\pi_B(k)})| \le \epsilon  & \text{~if~ }  k \in [s_i,e_i] \\
| dist(\mathbf{p}, \mathbf{a}_{\pi_A(i)}) - dist(\mathbf{p}, \mathbf{b}_{\pi_B(k)})| > \epsilon & \text{~otherwise~ }  \\
\end{split} 
\end{equation}
(See the red interval of Figure \ref{fig:proof})

Next, observe the equations \eqref{eq:condition} when i increases by 1 (i.e., $s_{i+1}$ and $e_{i+1}$).
See $| dist(\mathbf{p}, \mathbf{a}_{\pi_A(i+1)}) - dist(\mathbf{p}, \mathbf{b}_{\pi_B(j)})| \le \epsilon$ in the $min$ or $max$ expression. If we let $ c' = dist(\mathbf{p}, \mathbf{a}_{\pi_A(i+1)})$,  

\begin{equation}
\begin{split} 
& | c' - dist(\mathbf{p}, \mathbf{b}_{\pi_B(j)}) | \le \epsilon \\
\Rightarrow & c' - \epsilon\le dist(\mathbf{p}, \mathbf{b}_{\pi_B(j)}) \le  c' +  \epsilon \\
\end{split}
\end{equation}

Since $c \le c'$, it is also trivial that $c-\epsilon \le c' - \epsilon$ and $c+\epsilon \le c' + \epsilon$. Therefore, $s_i \le s_{i+1}$ and $e_i \le e_{i+1}$ (Compare the red line and the blue line in Figure \ref{fig:proof})
\end{proof}

According to Mathematical Expression \eqref{eq:se} in Lemma \ref{lemma:lem2}, given $\mathbf{a}_{\pi_A(i)}$, if $k \not\in [s_i, e_i]$, 
$| dist(\mathbf{p}, \mathbf{a}_{\pi_A(i)}) - dist(\mathbf{p}, \mathbf{b}_{\pi_B(k)})| > \epsilon$. According to Lemma \ref{lemma:lem1},  $dist (\mathbf{a}_{\pi_A(i)}, \mathbf{b}_{\pi_B(k)}) \ge | dist(\mathbf{p}, \mathbf{a}_{\pi_A(i)}) - dist(\mathbf{p}, \mathbf{b}_{\pi_B(k)})| > \epsilon$. Thus,  $dist (\mathbf{a}_{\pi_A(i)}, \mathbf{b}_{\pi_B(k)})  > \epsilon$. We do not need to evaluate $dist (\mathbf{a}_{\pi_A(i)}, \mathbf{b}_{\pi_B(k)})$ for  $k \not\in [s_i, e_i]$. Therefore, for $\mathbf{a}_{\pi_A(i)}$, we check similarity values for only $(e_i - s_i + 1)$ elements (i.e., $\mathbf{b}_{\pi_B(s_i)}$, 
$\mathbf{b}_{\pi_B(s_{i+1})}$, $\cdots$, $\mathbf{b}_{\pi_B(e_{i-1})}$ and $\mathbf{b}_{\pi_B(e_i)}$). 

However, we have to find $s_{i}$ and $e_{i}$ in the list $\mathbf{distB'}$. If we use the binary search, we can find two numbers (i.e., $s_{i}$ and $e_{i}$) with the time complexity O(log n) for each  $\mathbf{a}_{\pi_A(i)}$. However, we have to repeat that process $m$ times for $1, 2, \cdots, m$. That is, the total time complexity is O(m log n). Using Mathematical Expression \eqref{eq:se2} in Lemma \ref{lemma:lem2}, we can reduce the time complexity of evaluating $s_i$ and $e_i$ for $i \in \{1,2, \cdots, m\}$ to O(max(m, n)). Because $s_i \le s_{i+1}$ and  $e_i \le e_{i+1}$, we can find all $s_i$ and $e_i$ in the linear time by checking two lists, $dist(\mathbf{p}, \mathbf{a}_{\pi_A(i)})$ and $dist(\mathbf{p}, \mathbf{b}_{\pi_B(j)})$, sequentially. The algorithm to find $s_i$ and $e_i$ will be shown in Figure \ref{fig:algo2}.

\subsection{Fast Completion Algorithm}  \label{sec:main5}

Figure \ref{fig:algo1} shows the algorithm to efficiently evaluate the knowledge graph completion problem on a GPU. If the entity embedding vectors $E$ and the relation embedding vectors $R$ trained from the knowledge graph embedding model which is \emph{transformable to a metric space} are given, the algorithm returns triplet $(\mathbf{h, r, t})$ such that $dist_3(\mathbf{h,r,t}) \le \epsilon$. We assume that $connector_1$, $connector_2$ and $dist$ are given. We iterate the process in Lines 3-17 with respect to $\mathbf{r}$, relation embedding vector (Line 2). 

To transform the completion problem into the similarity join problem, we first get $A$ and $B$ using $Connector_1$ and  $Connector_2$ functions (Lines 3-4). $Connector_i(E, r)$ evaluates $connector_i(e,r)$ for each $e$ in $E$. $Connector_i()$ is processed using a GPU. Therefore, we can quickly evaluate it. After that, we compute $distA$ and $distB$ using $Dist()$ in Lines 5-6.  Pivot $p$ is properly set. Usually a zero vector can be used as a piviot. $Dist(x, Y)$ returns the results of $dist(x, y)$ for each $y\in Y$, where $x$ is a vector and $Y$ is a set of vectors. $Dist()$ is also processed using a GPU. Then we sort $distA$ and $distB$ (Lines 7-8) and rearrange $A$ and $B$ accordingly (Lines 9-10). In Line 11, we get $s_i$ and $e_i$ lists using the $compute\_range$ function which will be described below. 

\begin{figure}
\centering
\includegraphics[width=4.5in,page=7,trim=0 5in 0 0,clip]{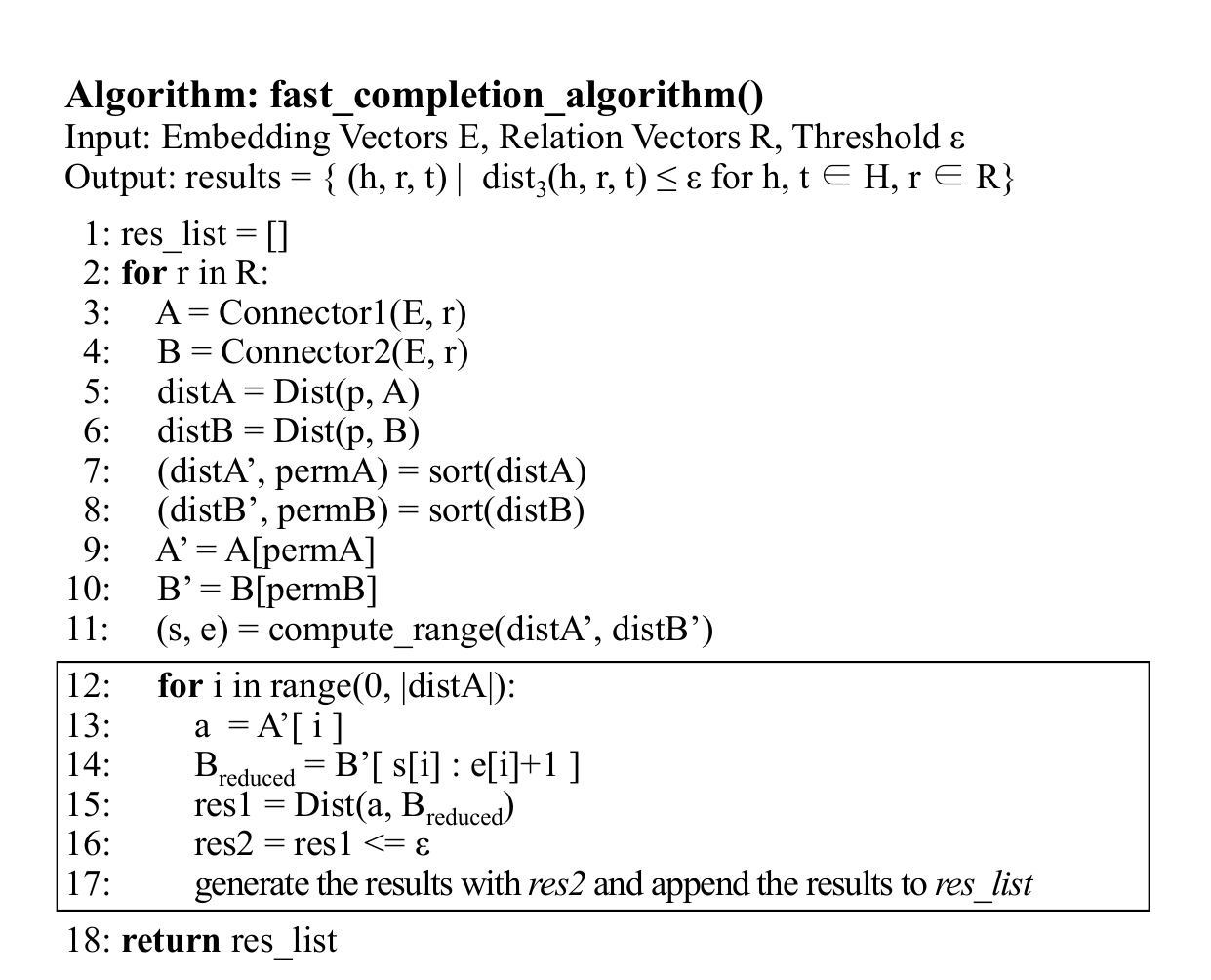}
\caption{Fast Completion Algorithm}
\label{fig:algo1}
\end{figure}

The process from Line 3 to Line 11 can be considered as a preprocessing step. We now perform the step of finding valid join pairs. In Lines 13-17, given $a_i$, we check if $dist(a_i, b_j)\le \epsilon$ for all $j \in \{1, 2, \cdots, n \}$. Because of Lemma \ref{lemma:lem2}, it would be all right to check if $dist(a_i, b_j)\le \epsilon$ for only $j \in \{s[i], s[i]+1, , \cdots, e[i]-1, e[i] \}$ (Lines 14 and 15). Note that $B'[start:end+1]$ is $B'[start],  B'[start+1], \cdots, B'[end]$ like the python operator. In the case of  $dist(a, b_j)\le \epsilon$, we generate the results (Line 17). However, the iteration is performed sequentially. In the next section, we will provide how to process this iteration in parallel.

\begin{figure}
\centering
\includegraphics[width=4.5in,page=8,trim=0 2.7in 0 0,clip]{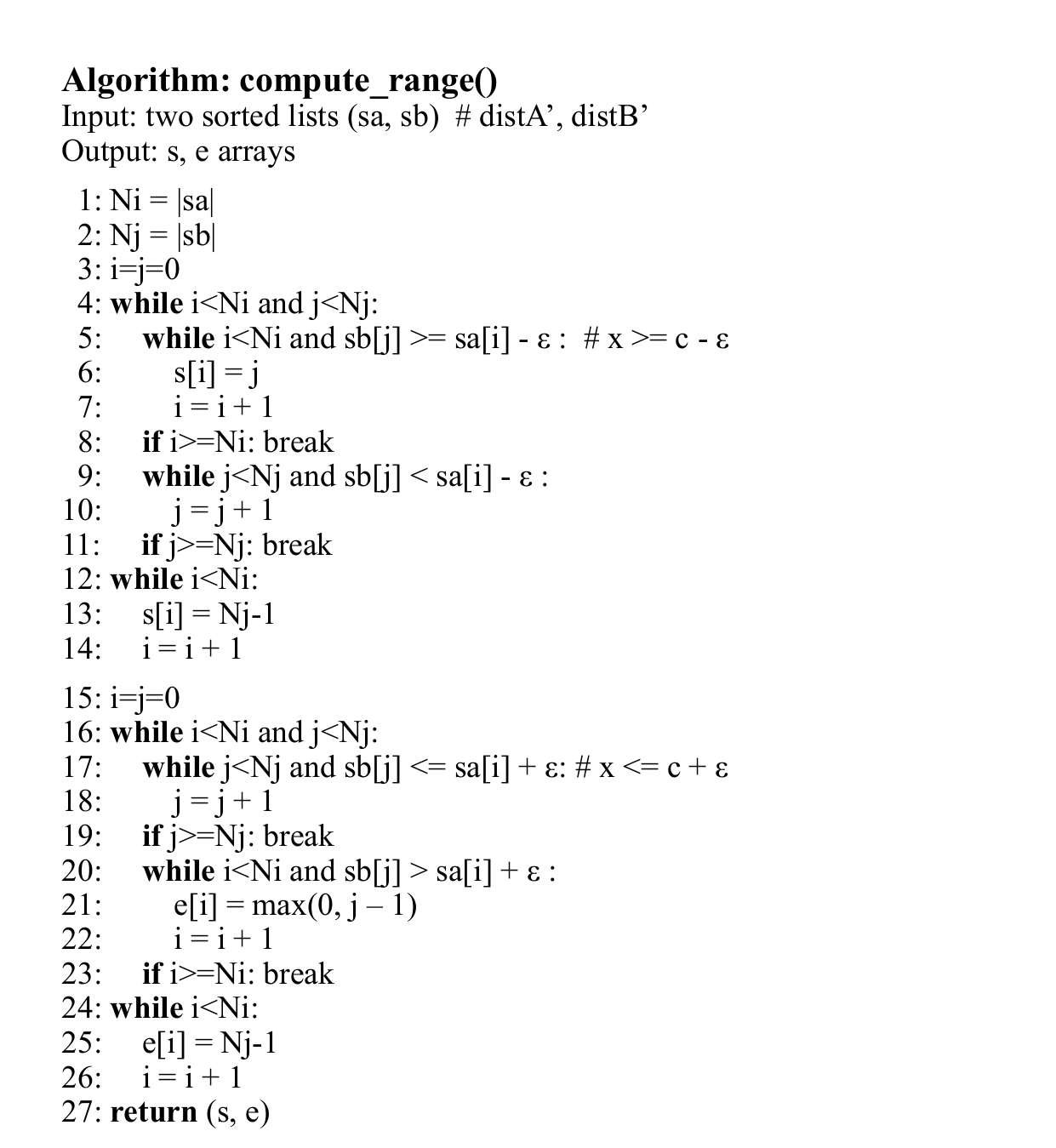}
\caption{Range Computation Algorithm}
\label{fig:algo2}
\end{figure}

The $compute\_range$ algorithm of Figure \ref{fig:algo2} is an algorithm to evaluate $s_i, e_i$ for $i \in [1,m]$ with the time complexity $O(max(m, n))$. The input of the algorithm is two sorted lists $distA'$ and $distB'$ which are named $sa$ and $sb$ in the algorithm. Based on Lemma \ref{lemma:lem2}, we can find $s_1, s_2, \cdots, s_m$ and $e_1, e_2, \cdots, e_n$ by scanning two sorted lists alternatively. The part from Line 1 to Line 14 is to evaluate $s_i$ while the part from Line 15 to Line 27 is to evaluate $e_i$. Two parts are performed in a similar way. First, we describe the part from Line 1 to Line 14.
To find the $s$ array, we execute two steps ($i$ pointer movement and $j$ pointer movement) alternately. We iterate the i pointer movement while $sb[j] \ge sa[i] - \epsilon$. In this step, we set $s[i]$ to the current j pointer (Line 6) because $sb[j] \ge sa[i] - \epsilon$ is satisfied for the first time with respect to $i$. See $c - \epsilon$ in Figure \ref{fig:proof} ($c=sa[i]$). Next, we iterate the j pointer movement by increasing j by 1 while $sb[j] < sa[i] - \epsilon$ (Lines 9-10). The i pointer movement and the j pointer movement are iterated alternately until either $i$ or $j$ pointer reaches the end (Lines 4-11). After breaking the loop, if $i < Ni$, we set $s[i]$ to the last index (Lines 12-14).

In a similar way, we can find $e_i$ (Lines 15-26). However, since $e_i$ is the last element among $\{ j | sb[j] \le c + \epsilon \}$ where $c=sa[i]$, we run the j pointer movement (Lines 17 and 18). We increase $j$ by 1 while $sb[j] \le c + \epsilon$. After that, we run the i pointer movement (Lines 20-22). In the i pointer movement, we set $e[i]$ to $j-1$ because in the previous step (j pointer movement), we can escape the loop when $sb[j] > sa[i] + \epsilon$. To avoid -1 in the index , we use $max(0, j-1)$. If there is no element between $c-\epsilon$ and $c + \epsilon$, the interval $[~s[i], e[i]~]$ should be rigorously the empty array. However, for programming convenience, in that case, the algorithm sets $s[i]$ and $e[i]$ to $0$ or sets $s[i]$ and $e[i]$ to $Nj-1$.

\subsection{Improving Parallelism} \label{sec:main6}

\begin{figure}
\centering
\includegraphics[width=4.5in,page=9,trim=0 4in 0 0,clip]{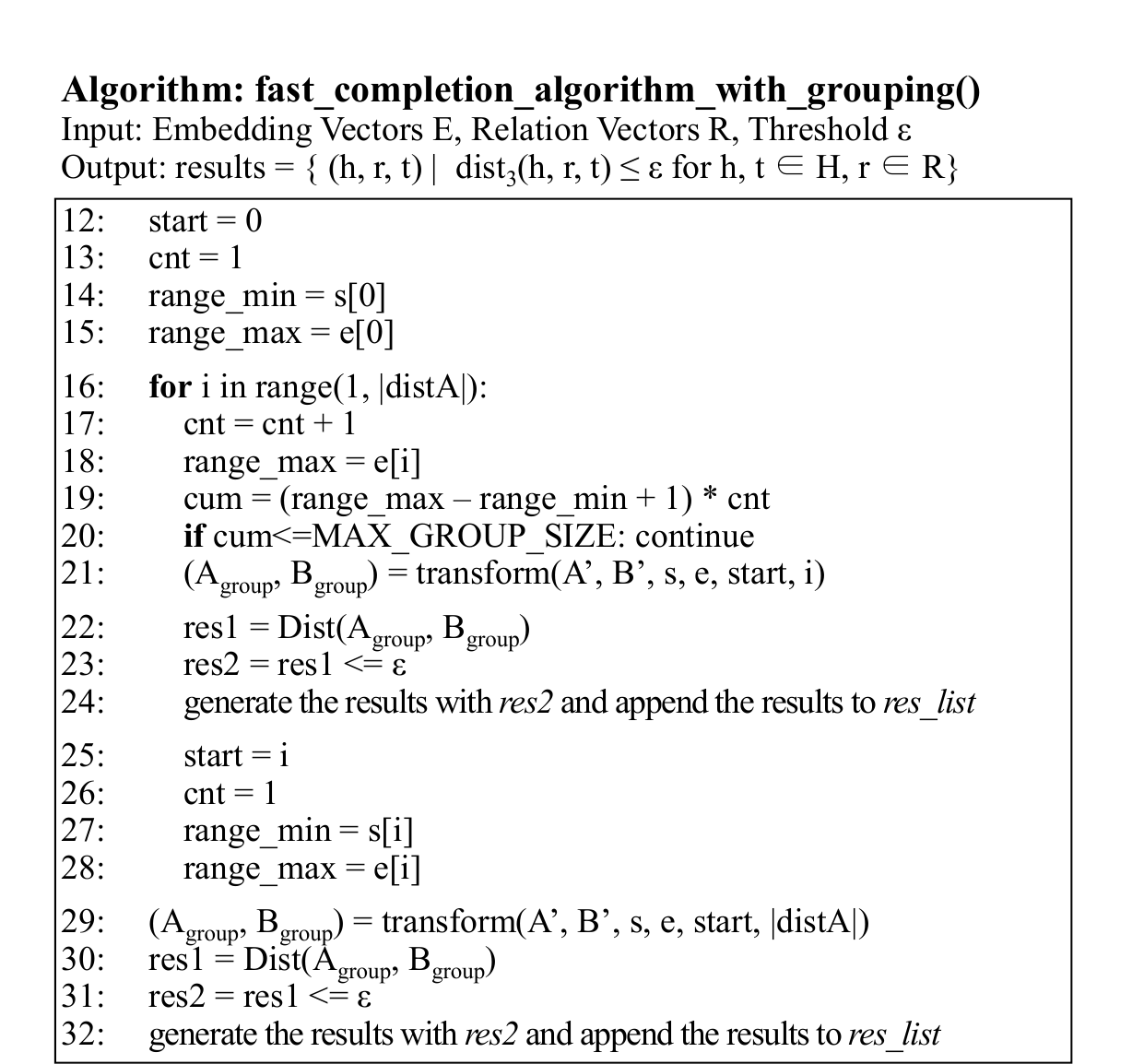}
\caption{Fast Completion Algorithm with Grouping}
\label{fig:algo3}
\end{figure}

\begin{figure}
\centering
\includegraphics[width=5in,page=10,trim=0 8.3in 0 0,clip]{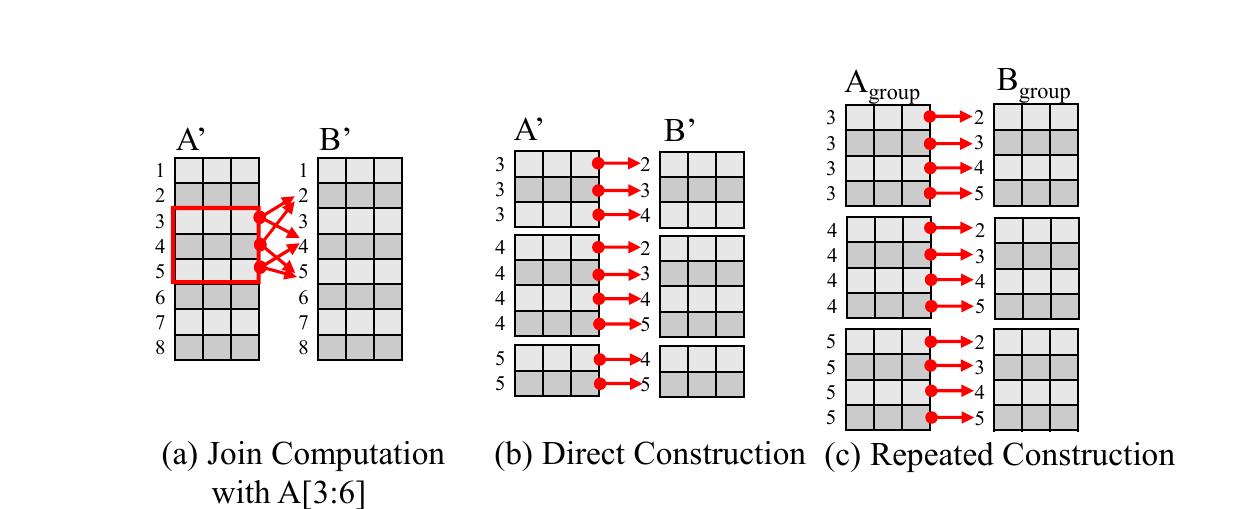}
\caption{Grouping Example}
\label{fig:grouping}
\end{figure}

Although we reduce the number of candidate pairs to be evaluated using $B_{reduced}$ in Line 14 of Figure \ref{fig:algo1}, we can improve the parallelism by grouping the process. Figure \ref{fig:algo3} shows the algorithm to more efficiently process multiple $Dist(a, B_{reduced})$ on a GPU. Since $B_{reduced}$ is a reduced form obtained by filtering out $B'$, the size of $B_{reduced}$ can be small. In that case, we can not have a lot of gain from a GPU. To avoid that, we can merge $B_{reduced}$ in multiple iterations for more simultaeneous processing. See Example-(a) of Figure \ref{fig:grouping}. We would like to compute $dist(a_i, B_{reduced})$ for $i \in \{3,4,5\}$. If s[3] = 2, e[3]=4, s[4] = 2, e[4]=5, s[5] = 4, e[5]=5, the algorithm of Figure \ref{fig:algo1} performs 3 group similarity evaluations (i.e., $dist(a_3, B[2:5])$, $dist(a_4, B[2:6])$, $dist(a_5, B[4:6])$) separately. For better parallel processing on the GPU, we can construct the data form like Figure \ref{fig:grouping}-(b) and evaluate them together. However, since building that kind of the data form can be burden on the GPU, we build the data form like Figure \ref{fig:grouping}-(c) by repeating the same vector set with the interval $range\_min$ and $range\_max$.
 
Now we explain the grouping algorithm of Figure \ref{fig:algo3}. The algorithm of Figure \ref{fig:algo3} corresponds to the part from Lines 12 to 17 in Figure \ref{fig:algo1} which is surrounded with the box.
For other parts. the algorithm with grouping uses the same code as the algorithm of Figure \ref{fig:algo1}. 
We first initialize $start$, $cnt$, $range\_min$ and $range\_max$ in Lines 12-15. Given the maximum group size (i.e., $MAX\_GROUP\_SIZE$), in Lines 16-20, we find the interval [start, i) such that the total size of $A_{group}$ fits $MAX\_GROUP\_SIZE$. The total size of $A_{group}$ is evaluated by $(max_{i \in \{start, start+1, \cdots, i-1 \} } e[i] - min_{i \in \{start, start+1, \cdots, i-1 \}  } s[i] +1) \times (i-start)$.
$cnt$, $range\_min$ and $range\_max$ correspond to $i-start$, $min_{i \in \{start, start+1, \cdots, i-1 \}  } s[i]$ and $max_{i \in \{start, start+1, \cdots, i-1 \} } e[i]$, respectively. Because $s_i \le s_{i+1}$ and $e_i \le e_{i+1}$ (Lemma \ref{lemma:lem2}), we can evaluate $range\_min$ and $range\_max$ more simply as shown in Lines 14 and 18.

After finding the interval $[start, i)$, we execute the transform() function to make the data form like Figure \ref{fig:grouping}-(c) (Line 21). In Line 21, the $transform()$ function makes the pair of $(A_{group}, B_{group})$ with  $A'[start:i]$ and $B'[range\_min:range\_max+1]$. Note that the notation of $X[i:j]$ follows the python syntax (i.e., X[i], X[i+1], $\cdots$, X[j-1]). After that, we compute Dist() and generate the results in Lines 22-24. Then, we initialize variables (Lines 25-28) and iterate the loop again. In Lines 29-32, the remaning part is evaluated in the same way above.

\begin{figure}
\centering
\includegraphics[width=4.5in,page=11,trim=0 8.3in 0in 0,clip]{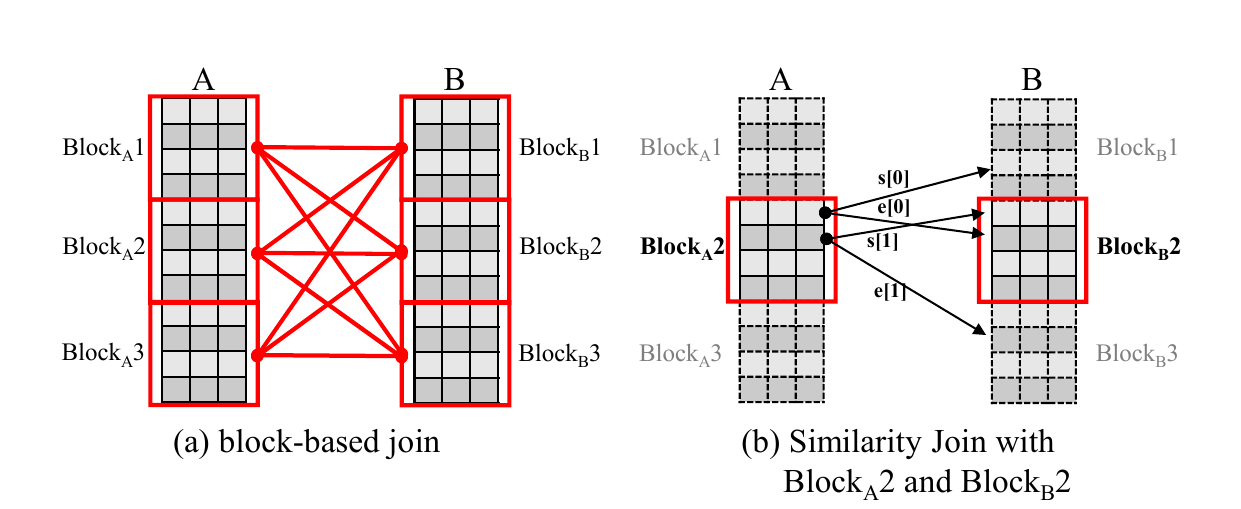}
\caption{Block Example}
\label{fig:blockex}
\end{figure}

\subsection{Partition-based Join} \label{sec:main7}
Until now, we assumed that $A$ and $B$ can be loaded into GPU memory. However, if $A$ and $B$ is large, it cannot be loaded into GPU memory. To solve it, we split data into the appropriate sized partitions and apply Algorithm \ref{fig:algo3} for all the block pairs as shown in Figure \ref{fig:blockex}-(a). In Figure \ref{fig:blockex}-(a), A and B have 3 blocks, respectively. We apply Algorithm \ref{fig:algo3} to 9 block pairs. However, $s[i]$ and $e[i]$ might be beyond the partition. 

For example, suppose that we apply Algorithm \ref{fig:algo3} to the $Block_A2$ and $Block_B2$ pair in Figure \ref{fig:blockex}-(b).  Consider the first and second elements in $Block_A2$. Because $s[0]$ is less than the start position of $Block_B2$, the index (s[0]) is invalid in $Block_B2$. In that case, we can set s[0] to the start position of $Block_B2$. Also, if e[1] in $Block_A2$ is greater than the end position of $Block_B2$, we can set e[1] to the end position of $Block_B2$. Using the concept, we can solve the invalid index problem in the partition-based join.

\section{Experiments} \label{sec:exp}
In this section, we will demonstrate our proposed framework can efficiently process the problem described in Section \ref{sec:problem}. The detailed experimental environment will be explained in Seciton \ref{sec:exp1} and the extensive experimental results will be shown in \ref{sec:exp2}.

\subsection{Experimental Environment} \label{sec:exp1}

To evaluate the performance of our framework in terms of the execution time, we used the commonly used data sets, \emph{WN18  \cite{kge:transe}, WN18RR  \cite{kge:conve}, FB15K \cite{kge:transe} and FB15K-237 \cite{fb15k237}}. WN18 and WN18RR data sets are extracted from WordNet \cite{wordnet} whereas FB15K and FB15K-237 data sets are extracted from the freebase knowledge graph \cite{freebase}. The number of nodes and the number of relation types in each data set are described in Figure \ref{fig:modelAcc}-(a). To evaluate our framework for those data sets, we downloaded four pretrained TransE models for WN18, WN18RR, FB15K, and FB15K-237 data sets at the LibKGE github website https://github.com/uma-pi1/kge. 

\begin{figure}
\centering
\includegraphics[width=3.5in,page=3,trim=0 5.2in 0 0,clip]{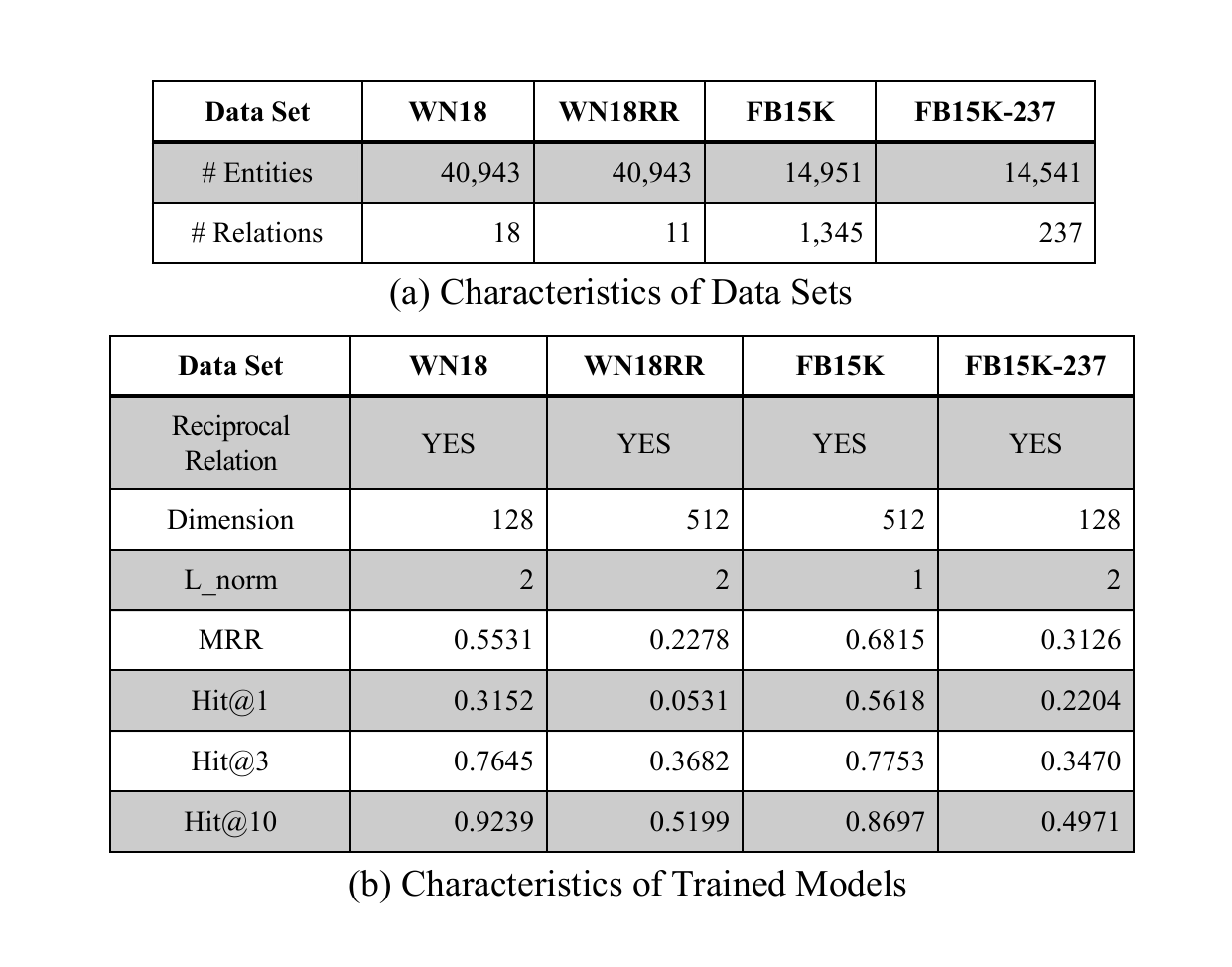}
\caption{Trained Model Characteristics and Model Accuracy}
\label{fig:modelAcc}
\end{figure}

Figure \ref{fig:modelAcc}-(b) shows the properties of four pretrained TransE models. All the models downloaded from LibKGE allow reciprocal relations. A model with reciprocal relations considers inverse relations. Therefore, the number of trained relation embedding vectors are 2 times more than the number of original relations. The accuracies of the trained models such MRR (Mean Reciprocal Rank) and Hits@k in the filtered setting are shown in Figure \ref{fig:modelAcc}-(b). Given the test triple data $Test$ and the total triple data $Total$ (i.e., train data+valid data+test data), MRR and Hits@k in the filtered setting are evaluated as follows:

\begin{equation}
\begin{split}
MRR &= \frac{\sum_i \frac{1}{rank_i} + \sum_i \frac{1}{rank'_i} }{2 |Test|}  \\
Hits@k &= \sum_i \frac{hits(rank_i, k) + hits(rank'_i, k)}{2|Test|}
\end{split}
\end{equation}
 
where $(h_i, r_i, t_i)$ is the $i$-th triplet of $Test$,  $rank_i$ is the order of $(h_i, r_i, t_i)$ in the sorted score values $\{score(h_i, r_i, y) | y \in E \text{~and~} (h_i, r_i, y) \not\in Total\}$ and $rank'_i$ is the order of  $ (h_i, r_i, t_i)$ in the sorted score values $\{score(x, r_i, t_i) | x \in E \text{~and~} (x, r_i, t_i) \not\in Total\}$ and the $hits$ function  is defined below.

\begin{equation*}
hits(rank, k) = \begin{cases}
                        1~\text{if rank $\le$ k} \\
                        0~\text{otherwise}
                  \end{cases}
\end{equation*}

Although the LibKGE github site (https://github.com/uma-pi1/kge) provides pretained models and their accuracies (MRR, Hits@1, Hits@3 and Hits@10), we reevaluated the accuracies of each pretrained model using the command \emph{kge test} given by LibKGE. This is because in some cases, the accuracies written in the site were slightly different from our evaluation.

Even though the number of entities in the data sets is tens of thousands, the number of records to be evaluated is much bigger. For instance, we have to check  $14,951 \times (1,345 \times 2) \times 14,951 = 601,302,158,690$ vector operations for the case of FB15K.  Therefore, those data sets are sufficient to evaluate our framework. For scalability test, we made several sized data sets by extracting the original data. We extracted only entities by the numbers $0.2\times|E|$, $0.4\times|E|$, $0.6\times|E|$, $0.8\times|E|$, $1.0\times|E|$ and keep the relations unchanged. Also, to show the impact of $\epsilon$, we used the following $\epsilon$ values for each data set. We set the starting point according to the Top-1 Value in Figure \ref{fig:top1} and increased the value by 0.5.

\begin{itemize}
\item{WN18:} 0.33, 0.83, 1.33, 1.83, 2.33 
\item{WN18RR:} 1.02, 1.52, 2.02, 2.52, 3.02
\item{FB15K:} 2.25, 2.75, 3.25, 3.75, 4.25
\item{FB15K-237:} 0.01, 0.51, 1.01, 1.51, 2.01 
\end{itemize}

As a comparison system, we used the Quickjoin (CPU) approach \cite{sim:quickjoin} and the naive GPU Approach. Quickjoin \cite{sim:quickjoin} efficiently processes the similarity jon problem in a metric space using a similar concept of quicksort. Quickjoin chooses two pivots $p_1$ and $p_2$ and partitions data into two groups, $G_1$ and $G_2$ by comparing $r(=dist(p_1, p_2))$ and $dist(p_1, x)$. During the partitioning, Quickjoin additionally makes two boundary groups. The \emph{Quickjoin(eps, distFunc, objs)} algorithm in \cite{sim:quickjoin} is to process the self similarity join problem. To process the similarity join problem instead of the self similarity join, we used the \emph{QuickjoinWin(eps, distFunc, objs1, objs2)} subfunction in \cite{sim:quickjoin}. Also, when the number of elements is small, Quickjoin performs a straightforward approach called \emph{NestedLoop} which requires two $for$ loops. We implemented Quickjoin using the python programming language and used $numpy$ as a basic data structure. 
To improve the performance of \emph{NestedLoop}, 
we chose the $for$ loop for the data set which has more elements and removed the $for$ loop by performing operations on two dimensional data (i.e., set of vectors) instead of one vector.

The naive GPU approach processes the knowledge graph embedding completion problem using a GPU. The pseudo code of the naive GPU approach is shown in Figure \ref{fig:naivegpu}. We implemented the naive GPU approach and our framework using the PyTorch framework in order to process GPU related tasks. To manage the size of GPU memory very tightly, we can use torch.cuda.empty$\_$cache(). However, in our experiments, we did not use torch.cuda.empty$\_$cache() in both the GPU-naive approach and our framework for better performance. We ran all the codes on a Dell PowerEdge R740 Server with two NVIDIA A100 GPUs (OS: Ubuntu 20.04). However, we used a single GPU in our experiments.

\begin{figure}
\centering
\includegraphics[width=4in,page=12,trim=0 6.5in 0 0,clip]{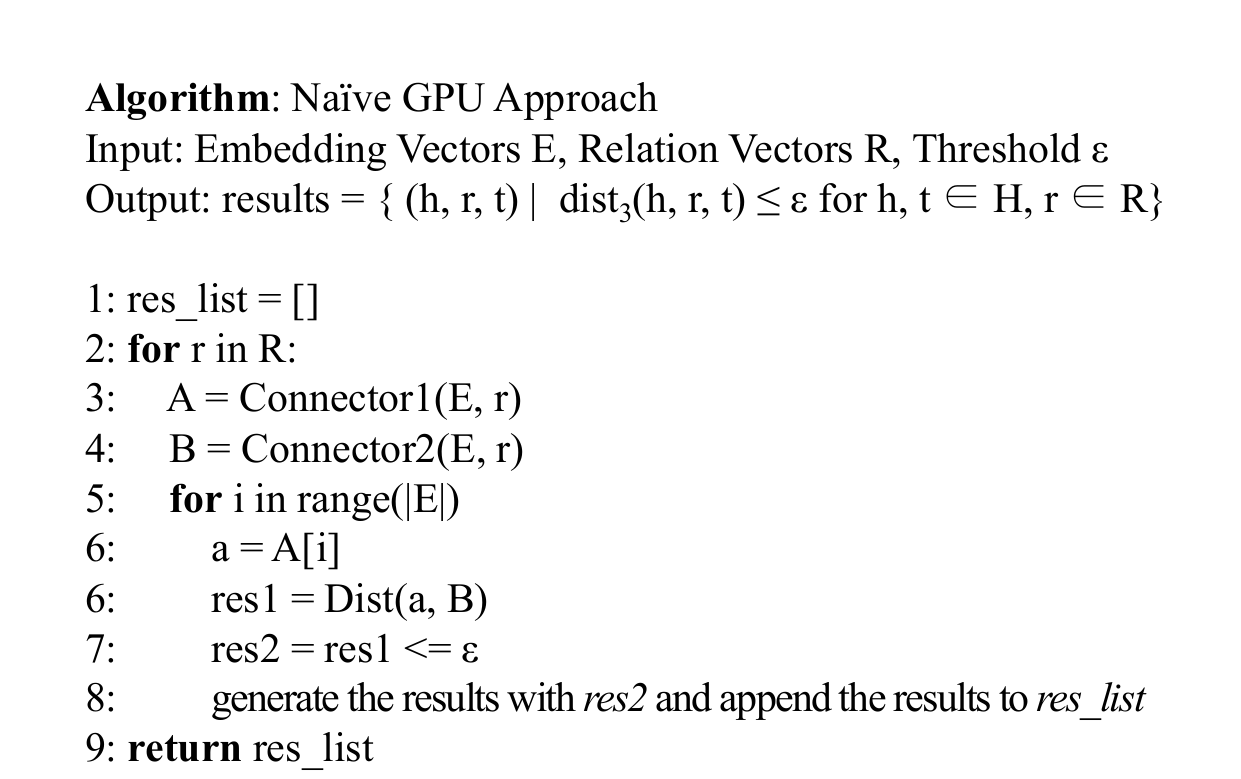}
\caption{Naive GPU Approach}
\label{fig:naivegpu}
\end{figure}

\subsection{Experimental Results} \label{sec:exp2}

\begin{figure}
\centering
\includegraphics[width=5in,page=13,trim=0 3in 0 0,clip]{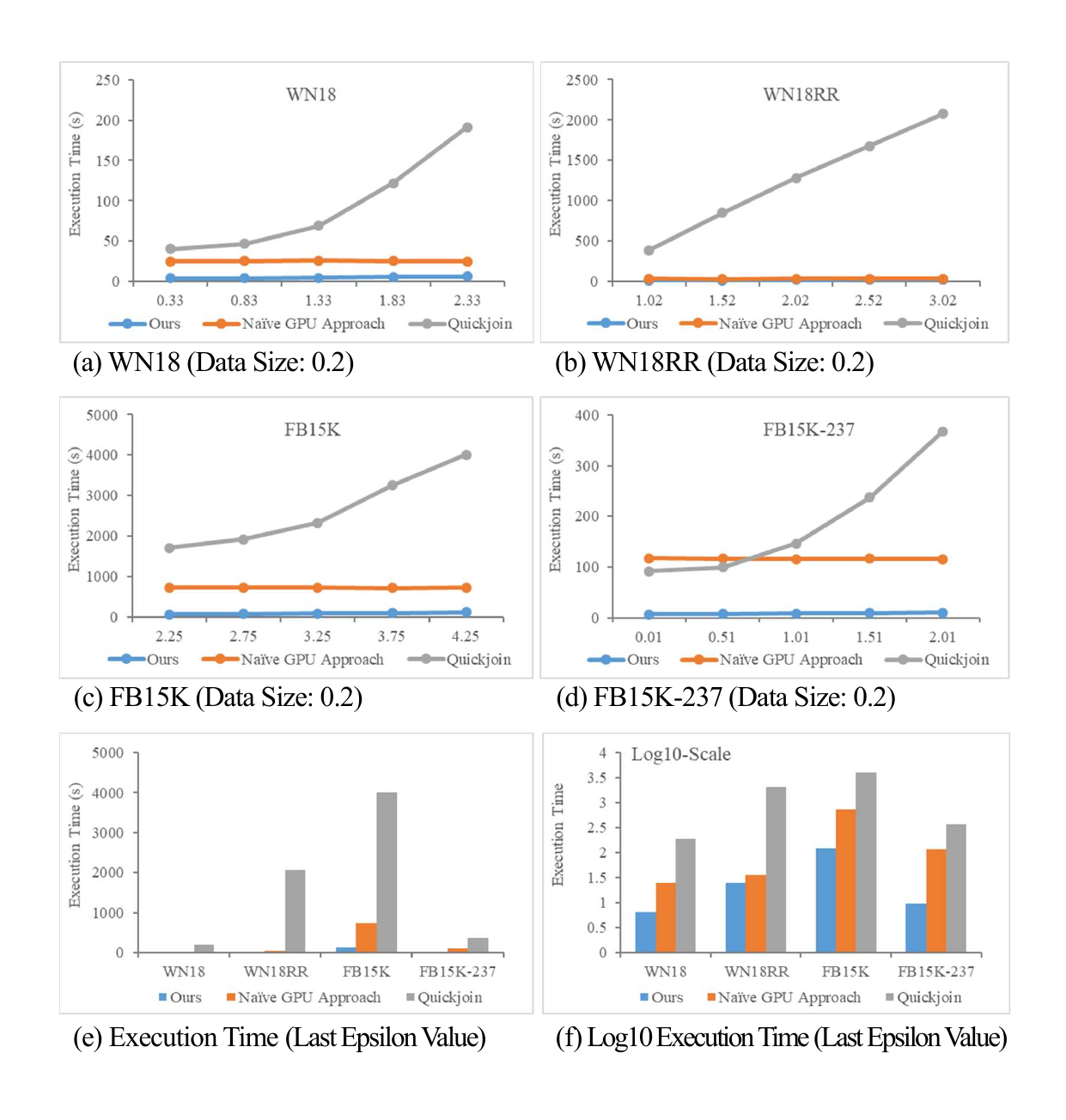}
\caption{Experimental Results for Our Framework, Naive GPU Approach and Quickjoin When Entity Size is 0.2 $|E|$}
\label{fig:exp1}
\end{figure}

We conducted extensive experiments with four trained models described above (i.e., {\footnotesize WN18, WN18RR, FB15K, FB15K-237}). We set the maximum group size (i.e., \emph{MAX$\_$GROUP$\_$SIZE}) in the algorithm of Figure \ref{fig:algo3} to 300,000. Figure \ref{fig:exp1} shows the processing time when the entity size is 0.2 $|E|$. We did experiments by changing the $\epsilon$ parameter. In Figures \ref{fig:exp1}-(a) to  \ref{fig:exp1}-(d), Quickjoin shows the worst performance in most cases while our framework shows the best performance in all the cases.
Although the Naive GPU approach is a straightforward method, it was much better than Quickjoin in terms of the execution time. This is because the similarity join problem can be easily broken down into multiple tasks and it is suitable to apply GPUs.

\begin{figure}
\centering
\includegraphics[width=5in,page=14,trim=0 6in 0 0,clip]{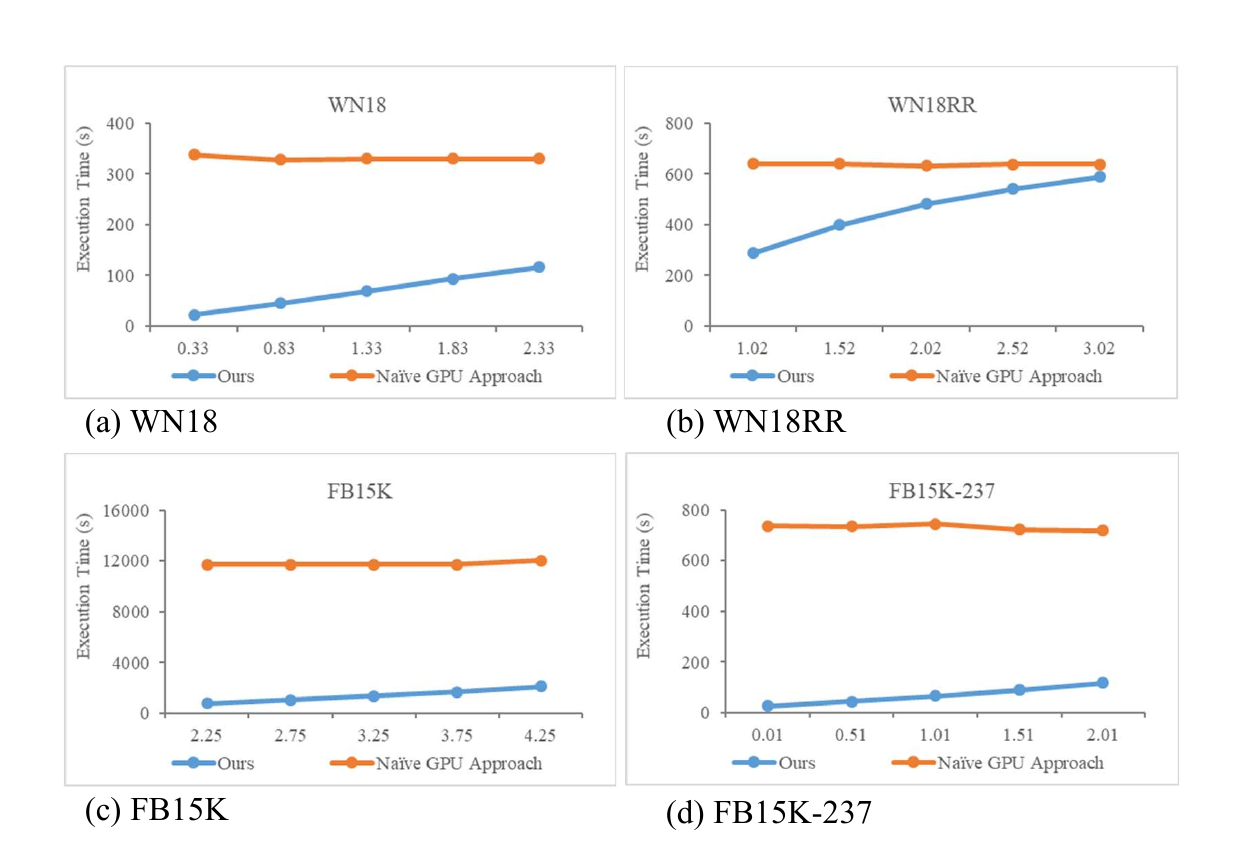}
\caption{Experimental Results according to $\epsilon$}
\label{fig:exp2}
\end{figure}

Figure \ref{fig:exp1}-(e) shows the execution times for four models when the $\epsilon$ parameter is set to the last $\epsilon$ value. Because Quickjoin is significantly slower than other approaches, we illustrate the processing time in log10 scale in Figure \ref{fig:exp1}-(f). As shown in Figures \ref{fig:exp1}-(e) and \ref{fig:exp1}-(f), Quickjoin has significantly worse performances compared to our framework and the naive GPU approach. On average, our framework is 7.4 times faster than the Naive GPU approach and 31.7 times faster than Quickjoin in the experiments of Figures \ref{fig:exp1}-(a) to (d). From the following experiments, we did not include the Quickjoion approach because of its slow performance. Also, we did not use a partition-based join in our framework because all of the entities can be sufficiently loaded into GPU memory. In the algorithm of Figure \ref{fig:algo1} with TransE embedding, $distB$, $distB'$ and $B'$ (Lines  6, 8 and 10) are the same regardless of $r$. Therefore, we evaluated Lines 6, 8 and 10 only once in the real implementation. 

Figure \ref{fig:exp2} shows experimental results according to $\epsilon$ when the number of entities is 1.0 $|E|$. In most cases, our framework is significantly better than the naive GPU approach. The performance of the naive GPU approach is not affected by the $\epsilon$ paratermeter because it has to check all the pairs regardless of the $\epsilon$ parameter. On average, our framework is 8.0 times faster than the naive GPU approach in Figure \ref{fig:exp2}. Also, in the case of FB15K-237 with $\epsilon=0.01$, our framework is 28.8 times better than the GPU naive approach.
 
Figure \ref{fig:exp3} shows experimental results according to the data size. As mentioned earlier, we used five sized data sets by controlling the number of entities, 0.2 $|E|$, 0.4 $|E|$, 0.6 $|E|$, 0.8 $|E|$, and 1.0 $|E|$. We did not change the size of relation types but used the same for five cases. The $\epsilon$ parameter was set to the middle point. Although the size of entities is not big, the number of records to be checked is big because we have to check $N \times R \times N $ records, where $N$ is the number of entities and $R$ is the number of relation types. As expected, the execution time increases as the number of entities gets higher in both our framework and the naive GPU approach. Our framework shows a much better performance compared to the naive GPU approach in most cases. On average, our framework is 6.8 times quicker than the naive GPU approach in Figure \ref{fig:exp3}.

To show the impact of the maximum group size (i.e., \emph{MAX$\_$GROUP$\_$SIZE}) in the algorithm of Figure \ref{fig:algo3}, we ran experiments by varying the maximum group size. Figure \ref{fig:expgroup} shows the experimental results according to the various group sizes. In general, as the group size increases, the performance gets better in terms of the execution time. However, if the maximum group size is too large, the performance might be worse as shown in FB15K and FB15K-237. We need to control the maximum group size properly according to the data set and the computing environment. 

\begin{figure}
\centering
\includegraphics[width=5in,page=15,trim=0 6in 0 0,clip]{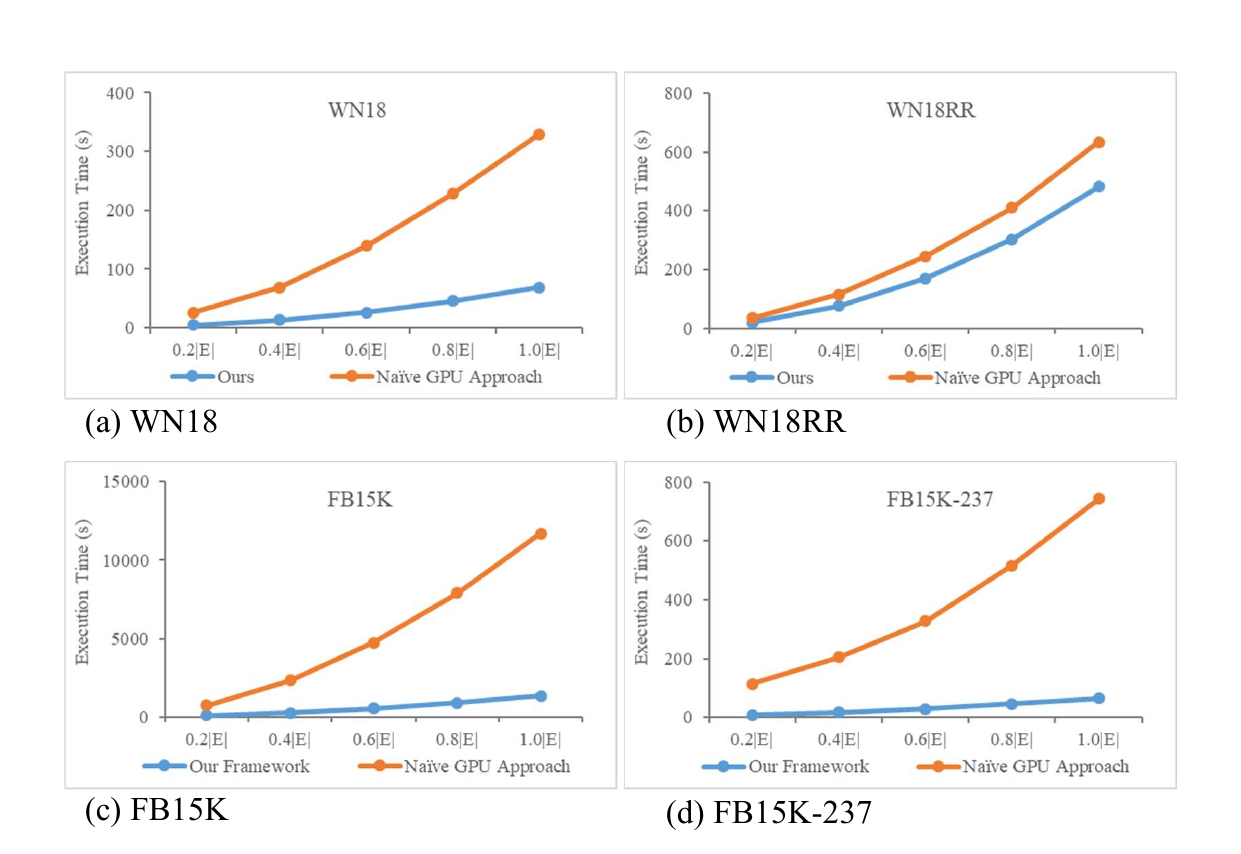}
\caption{Experimental Results according to Data Size}
\label{fig:exp3}
\end{figure}

\begin{figure}
\centering
\includegraphics[width=5in,page=16,trim=0 6in 0 0,clip]{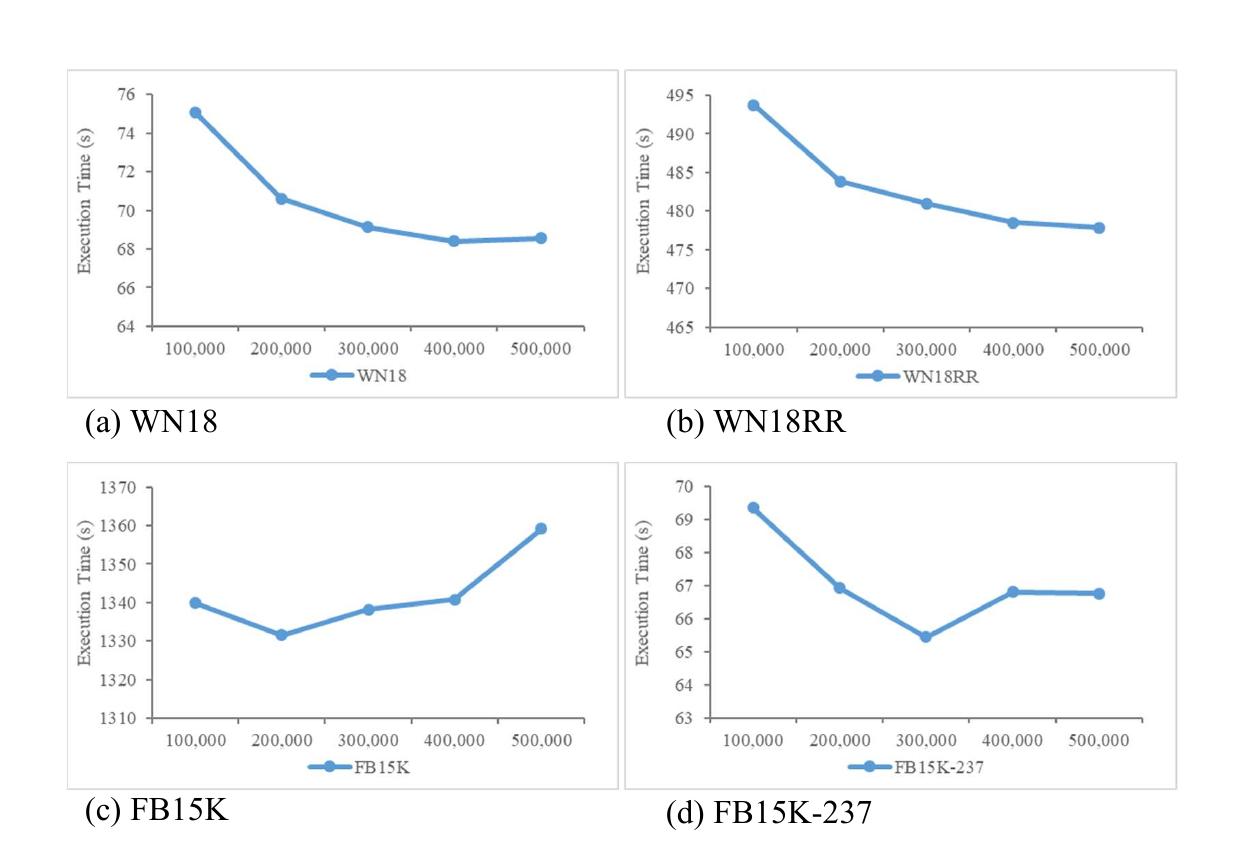}
\caption{Experimental Results according to Group Size}
\label{fig:expgroup}
\end{figure}

\section{Conclusion} \label{sec:con}
In this paper, we considered the knowledge graph embedding completion problem in terms of the running time. We developed a systematic framework to quickly process the knowledge graph embedding completion problem by focusing on the knowledge graph embedding models which are \emph{transformable to a metric space}. In the framework, we transformed the knowledge graph embedding completion problem into the similarity join problem. After that, based on formulas derived from the properties of a metric space, we devised an efficient algorithm to process the similarity join problem on GPUs. Finally, we showed that our framework was much better than the naive GPU approach and the CPU-based QuickJoin.

\section*{ACKNOWLEDGMENTS}
This work was partly supported by  Electronics and Telecommunications Research Institute(ETRI) grant funded by the Korean government (No.22ZS1100, Core Technology Research for Self-Improving Integrated Artificial Intelligence System) and Institute of Information \& communications Technology Planning \& Evaluation (IITP) grant funded by the Korea government(MSIT) (No.2022-0-00907, Development of AI Bots Collaboration Platform and Self-organizing AI).

\bibliographystyle{unsrt}  

\end{document}